\documentclass[lettersize,journal]{IEEEtran}
\usepackage{amsmath,amsfonts}
\usepackage{algorithmic}
\usepackage{algorithm}
\usepackage{array}
\usepackage[caption=false,font=normalsize,labelfont=sf,textfont=sf]{subfig}
\usepackage{textcomp}
\usepackage{stfloats}
\usepackage{url}
\usepackage{verbatim}
\usepackage{graphicx}
\usepackage{cite}
\usepackage{amsthm}
\usepackage{amssymb}
\usepackage{booktabs}
\newtheorem{theorem}{Theorem}
\AtBeginEnvironment{theorem}{\setlength{\parindent}{0pt}}
\AtEndEnvironment{theorem}{\setlength{\parindent}{15pt}}
\newtheorem{corollary}{Corollary}
\AtBeginEnvironment{corollary}{\setlength{\parindent}{0pt}}
\AtEndEnvironment{corollary}{\setlength{\parindent}{15pt}}
\usepackage{diagbox}
\usepackage{xcolor}

\begin{document}

\title{Reinforcement Learning for \\  Game-Theoretic Resource Allocation on Graphs}

\author{Zijian An,~\IEEEmembership{Student Member,~IEEE} and Lifeng Zhou*\thanks{* Corresponding author},~\IEEEmembership{Member,~IEEE}
\thanks{The authors are with the Department of Electrical and Computer Engineering, Drexel University, PA 19104 USA (email: za382@drexel.edu, lz457@drexel.edu).}
}


\maketitle

\begin{abstract}

Game-theoretic resource allocation on graphs (GRAG) involves two players competing over multiple steps to control nodes of interest on a graph, a problem modeled as a multi-step Colonel Blotto Game (MCBG). Finding optimal strategies is challenging due to the dynamic action space and structural constraints imposed by the graph. To address this, we formulate the MCBG as a Markov Decision Process (MDP) and apply Reinforcement Learning (RL) methods, specifically Deep Q-Network (DQN) and Proximal Policy Optimization (PPO). To enforce graph constraints, we introduce an action-displacement adjacency matrix that dynamically generates valid action sets at each step. We evaluate RL performance across a variety of graph structures and initial resource distributions, comparing against random, greedy, and learned RL policies. Experimental results show that both DQN and PPO consistently outperform baseline strategies and converge to a balanced $50\%$ win rate when competing against the learned RL policy. Particularly, on asymmetric graphs, RL agents successfully exploit structural advantages and adapt their allocation strategies, even under disadvantageous initial resource distributions.

\vspace{2mm}
\textit{Note to Practitioners:} This work presents an RL-based framework for solving the GRAG, with a particular focus on DQN and PPO as the core learning algorithms. 
The framework is designed to compute optimal strategies for resource allocation in competitive settings constrained by graph topology. The methodology and findings of this study have practical relevance across a wide range of domains where strategic resource allocation under spatial or network constraints is essential. Applications include autonomous surveillance and reconnaissance, where robots gather information in the presence of adversaries; cybersecurity, where attackers and defenders allocate limited resources to protect or compromise sensing and communication assets; and financial markets, where firms strategically bid for customer segments to gain competitive advantages. This framework is especially valuable for practitioners in robotics, control, and networked systems operating in contested, dynamic environments, where robust and adaptive strategies are essential.

\end{abstract}

\section{Introduction}
Recent advances in computing, sensing, and communication technologies have reinforced the central role of resource allocation across a broad spectrum of engineering and automation applications \cite{Ballotta2022ToCO, Fernando2024EnergyEfficientII}. Many areas inherently involve the competitive strategic distribution of limited resources under spatial and structural constraints. Such competitive and resource-constrained scenarios appear across a broad range of domains. In surveillance and reconnaissance tasks~\cite{grocholsky2006cooperative, thakur2013planning, Rabban2020ImprovedRCA, Zhang2021GameTSA}, robotic agents must be strategically deployed to monitor expansive areas while responding to the presence of adversarial forces or contested terrain. In cybersecurity~\cite{Jia2018StackelbergGAA,Alouini2024JammingIIA, Zou2022EquilibriumAAA, Sen2010AnESA, Wilden2022CyberNRA}, defenders and attackers compete by allocating limited resources to either preserve or disrupt network integrity. Similarly, competitive market bidding~\cite{Wang2023HiBidACA, Zhang2023FairnessAwareCBA} involves budget-constrained companies strategically distributing marketing efforts or pricing incentives across customer groups to maximize influence and market share.

The Colonel Blotto Game (CBG) has been widely adopted to model these competitive resource allocation scenarios. Originally introduced by Borel~\cite{borel1921theorie} and further developed in~\cite{borel1953theory}, the CBG captures zero-sum strategic scenarios where two players allocate limited resources across multiple battlefields to maximize territorial gains. Traditionally, the CBG assumes a one-shot interaction, reflecting instantaneous strategic decisions, but its principles naturally extend to sequential or spatially structured competitive environments encountered in real-world automation and control applications.
\begin{figure}
    \centering
    \includegraphics[width=0.9\linewidth]{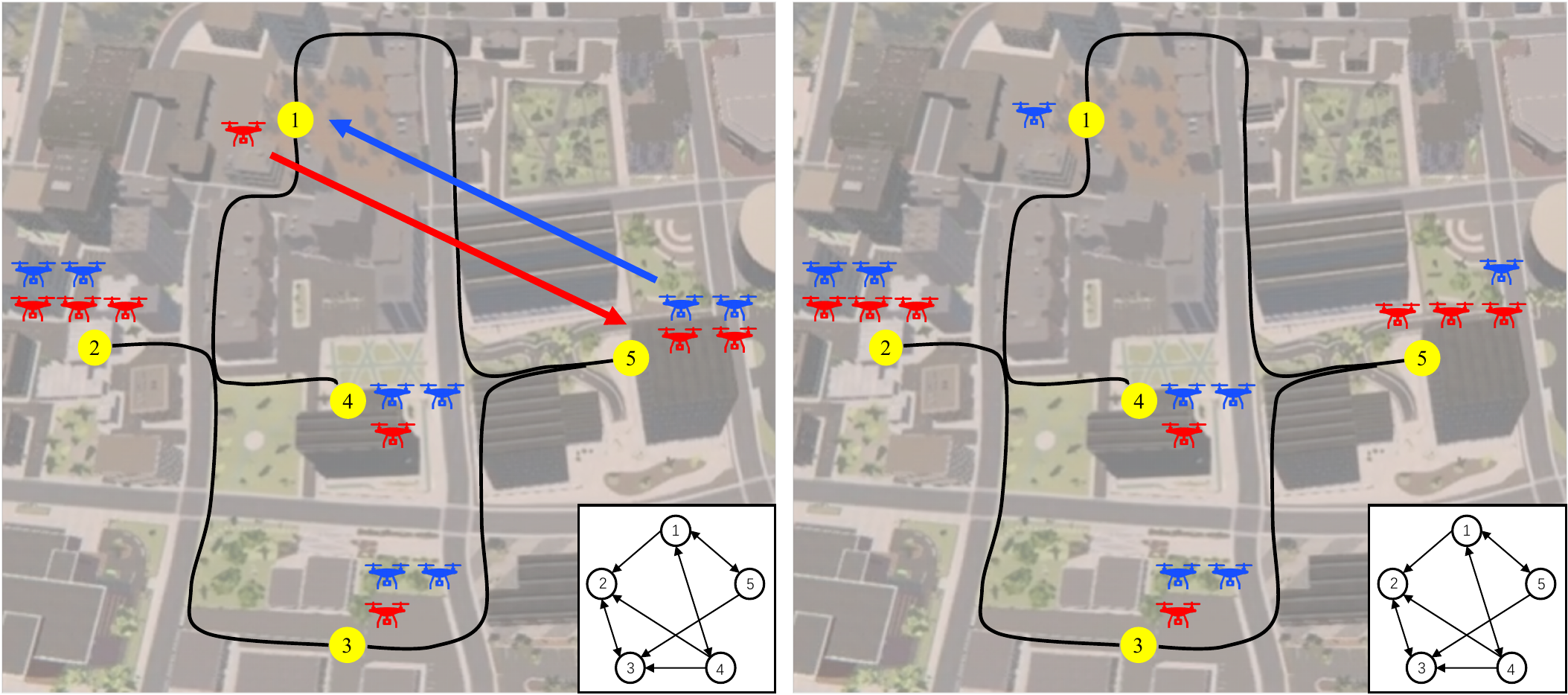}
    \caption{GRAG on a five-node graph involves two players, the red and blue, each with eight units of resources. The players make moves simultaneously. In the left subfigure, the game begins in a tied status. The red has a clear winning strategy: moving a resource from node 1 to node 5. However, the blue player is aware of this strategy and can counter it by moving one resource from node 5 to node 1. As shown in the right subfigure, this counteraction restores a tie, and the game continues and evolves to the next step.
    }
    \label{fig:illus_0}
\end{figure}
In~\cite{shishika2022dynamic} and \cite{an2023double}, the conventional CBG is extended into the CBG on a graph, where two players compete over a networked battlefield. This formulation introduces a graphical constraint that limits resource movement to adjacent nodes. As a result, the game evolves from a single-step allocation problem into a multi-step dynamic process. Starting from an initial resource distribution, both players strategically maneuver their resources across connected nodes over multiple steps until one side secures control of the majority. Due to its ability to capture both the spatial structure and the sequential nature of competitive interactions, the MCBG on graphs serves as a natural and expressive model for GRAG, as illustrated in Figure~\ref{fig:illus_0}.

Reinforcement Learning (RL) has proven to be a powerful tool for computing optimal policies in strategic games. A notable example is the game of Go, where the RL agent AlphaGo Zero achieved superhuman performance through self-play training~\cite{schrittwieser2020mastering,silver2017mastering,silver2018general,lai2015giraffe}. Beyond Go, RL has been applied to various other competitive settings. For instance, \cite{madani2014game} introduced a game theory-reinforcement learning (GT-RL) approach for fair and efficient benefit allocation in multi-operator, multi-reservoir systems. Similarly, \cite{wender2012applying} evaluated RL algorithms, particularly Q-learning and Sarsa variants, for micromanaging combat units in StarCraft: Brood War, demonstrating their effectiveness in optimizing unit control and decision-making in dynamic environments. Additionally, \cite{ye2020towards} proposed a deep RL-based learning paradigm for Multiplayer Online Battle Arena (MOBA) games, successfully training AI agents that outperformed top esports players in Honor of Kings.
Beyond gaming applications, RL has been widely employed in resource allocation problems. \cite{ye2019deep} developed an RL-based mechanism for vehicle-to-vehicle communication, while \cite{cui2019multi} introduced a multi-agent RL framework for dynamic resource allocation in UAV-enabled networks. In \cite{deng2020dynamical}, an RL-based resource allocation system for edge computing optimized service trustworthiness by modeling the resource adjustment process as a Markov Decision Process (MDP) and training policies to maximize service-level agreement (SLA) compliance. Furthermore, \cite{noel2022reinforcement} applied RL to compute optimal strategies for a conventional Blotto game with three nodes and ten resources, achieving a $60\%$ win rate against a random-policy agent.

Building on these advancements, we present the first RL framework designed to solve the GRAG. Specifically, we model GRAG as an on-graph CBG, which we further formulate as a Markov Decision Process (MDP).
A fundamental challenge in leveraging learning-based algorithms for the  GRAG is the principled design of action spaces, state abstractions, and reward structures that guide the learning process toward effective and generalizable policies. An overly detailed state representation may lead to high-dimensional input spaces, increasing training complexity and the risk of overfitting. In addition, overly abstract representations may degrade learning performance due to insufficient information. To this end, we propose a tailored state representation that enables efficient training and improved learning performance. Another challenge in applying RL to  GRAG is that while Deep Neural Networks (DNNs) in RL operate with a fixed structure, the action space varies dynamically at each step due to the graph constraints. To address this, we introduce an action-displacement adjacency matrix to determine valid actions at each step, enabling both players to effectively navigate the dynamic action space.
For evaluation, we assess the strategies trained with DQN and PPO under various initial conditions and across different graph structures, including an analysis of their generalization capabilities. The results demonstrate that both RL-based strategies achieve high win rates, with the DQN-trained strategy exhibiting a certain degree of generalization.

\section{Background}
\subsection{On-Graph Colonel Blotto Game}
\label{subsec:cbg}
First, we model GRAG as an on-graph CBG. A CBG is a two-player constant-sum game in which players allocate resources across several strategic nodes simultaneously \cite{borel1921theorie}. This concept can be extended to the CBG on a graph \cite{shishika2022dynamic,an2023double}, where not all nodes are directly connected. Due to the constraint imposed by the graph structure, resources can only be transferred between connected nodes, which restricts the players' strategies (i.e., action space).
In this paper, we assume that each player has $M$ units of resources. The graph is represented by $G=\{\mathcal{V}, \mathcal{E}\}$, where $\mathcal{V}$ is the node set and $\mathcal{E}$ is the edge set. The graph $G$ contains $N$ nodes, denoted as $|\mathcal{V}|=N$. The connectivity of graph $G$ is described by the adjacency matrix $\mathbf{H}$, where $\mathbf{H}_{ij}=1$ if $(i,j)\in\mathcal{E}$, $\mathbf{H}_{ij}=0$ otherwise. The player resource distribution is represented by $\mathbf{d}$. Each node is assigned a unique ID ranging from 0 to $N-1$, as shown in Figure~\ref{fig:bias}. This ID will be used in Section~\ref{subsec:setup}.

It is important to note that on-graph CBG is a multi-step game, while conventional CBG is a single-step game. In conventional CBG, there is no graph constraint, allowing any desired resource allocation to be achieved in a single step, and the outcome is determined permanently. However, in an on-graph CBG, the graph constraint means that a resource distribution may require multiple steps to achieve or might not be achievable at all. For example, if one wants to move resources from node $i$ to node $j$, but $\mathbf{H}_{ij}=0$ and $\mathbf{H}_{ik}=\mathbf{H}_{kj}=1$, the transfer can only be completed in two time steps: first by moving resources from node $i$ to node $k$, and then from node $k$ to node $j$. This highlights that on-graph CBG explicitly inherits the property of multi-step resource transition.
Therefore, as a multi-step game, on-graph CBG begins with initial resource distributions $\mathbf{d}_0^1, \mathbf{d}_0^2$ for both players. Starting from the initial state, players adjust their resource allocations across connected nodes. The current allocation affects the possible allocations in the next step. This requires calculating the valid allocation space (referred to as the action space hereafter) at each game step. A primary challenge in applying RL to on-graph CBG is the fixed structure of the DNN conflicting with the dynamically changing valid action space at each step. The dimension of the DNN output is pre-defined in RL. In DQN, the outputs are each action's Q value given the current state, and in PPO, the outputs are the probability of taking each action given the current state. However, the graph constraint results in a dynamically changing action space, which in turn necessitates a flexible output dimension.
To resolve this, we introduce a novel action-displacement adjacency matrix that selectively filters the DNN's output, ensuring that only valid actions are considered at each game step.


Moreover, since the distribution at the next step only depends on the distribution at the current step, but not the distributions at previous steps, which indicates that the on-graph CBG satisfies the Markov property and thus can be modeled as an MDP.

\subsection{Reinforcement Learning for  GRAG}
Since both players execute actions simultaneously from distinct action spaces, which, due to the graph constraint, vary with the current resource distribution, the state transition function $P(\mathbf{s}_{t+1}|\mathbf{s}, \mathbf{a})$ for an individual player cannot be explicitly expressed. 
Therefore, we adopt model-free reinforcement learning approaches: the value-based method Deep Q-Network (DQN) and the policy-gradient method Proximal Policy Optimization (PPO) to address the  GRAG.

\textbf{DQN} is a widely used model-free value-based RL approach \cite{mnih2015human}. 
The core of DQN lies in approximating the traditional Q-value function using a deep neural network, known as the Q-Network. The Q-value function $Q(\mathbf{s}, \mathbf{a})$ represents the expected cumulative reward obtained by taking action $\mathbf{a}$ from state $\mathbf{s}$. The primary objective of DQN is to find a Q-value function such that the Q-values under the current policy closely approximate the true expected rewards. We describe how we apply the DQN framework for solving the  GRAG in Section~\ref{subsec:dqn}.

\textbf{PPO} is a popular model-free policy gradient RL approach \cite{schulman2017proximal}. The objective of PPO is to optimize the policy function $\pi_\theta(\mathbf{a}|\mathbf{s})$ to maximize the cumulative reward. It is fundamentally based on the policy gradient method. When updating the policy parameter $\theta$, PPO employs a clipped objective function to stabilize the optimization process, restricting the extent of policy updates and preventing large policy changes that could lead to instability. We describe how we apply the PPO framework for solving the GRAG in Section~\ref{subsec:ppo}.

\section{Approach}
We begin by outlining the RL framework for solving the GRAG, including key conventions such as player states, observations, rewards, and actions. Additionally, we define the game progression mechanics and the conditions for determining the winner (Section~\ref{subsec:setup}).
Since the graph may not always be fully connected, the set of valid actions at each step is constrained and dynamically changes based on the player's current distribution. This necessitates a well-structured action mechanism to generate a dynamic action set. To address this, we introduce our approach for constructing the valid action set in Section~\ref{subsec:mask}.
Building on this foundation, we describe the training processes for DQN and PPO in Sections~\ref{subsec:dqn} and \ref{subsec:ppo}, detailing how these methods learn optimal allocation policies.
\subsection{RL Conventions}
\label{subsec:setup}
We consider a CRAG on an $N$-node graph with each player owning a total of $M$ resources. The CRAG (CBG) can be modeled as a Markov Decision Process $\text{MDP}:=(\mathcal{S}, \mathcal{A}, \mathcal{P}, R, \gamma)$ with $\mathcal{S}$ the state space, $\mathcal{A}$ the action space, $\mathcal{P}$ the state transition probability, $R$ the reward, and $\gamma$ the discount factor. We explain these conventions in detail as follows.
\begin{itemize}
    \item $\mathbf{d}_i\in\mathcal{D}_i, i\in\{1,2\}$ denotes the resource distribution of each of the two players. For an $N$-node graph, $\mathbf{d}_i$ is an N-dimensional vector where each element represents the amount of resources allocated to the corresponding node.
    \item $\mathcal{S}:=\{\mathbf{s}\in\mathbb{R}^N|\ \mathbf{s}=\mathbf{d}_1-\mathbf{d}_2, \mathbf{d}_i\in\mathcal{D}_i\}$ is the state space. The motivation for defining the state as the resource difference between two players is twofold. First, the state includes information about the resource distributions of both players. It is essential in a two-player game that the state reflects shared information from both players. Decisions made based on this shared information are more reasonable and balanced, whereas relying on only one player’s distribution would lead to incomplete or biased action selection. Second, compared to directly including the two players' resource distributions, $(\mathbf{d}_t^1, \mathbf{d}_t^2)$ as the state, using $\mathbf{d}_t^1-\mathbf{d}_t^2$ reduces the dimensionality of the state from $|\mathcal{D}^2|$ to $|\mathcal{D}|$. This dimensionality reduction significantly accelerates the learning process. In both DQN and PPO, the state serves as the network input. High-dimensional inputs can impose excessive demands on the depth of the network, which may lead to longer convergence times or even overfitting, ultimately degrading training performance.
    \item $\mathcal{A}:=\{\mathbf{a}\in \mathbb{R}^M\ |\ a_i\in\mathcal{I}(N)\}$ denotes each player's action space with $\mathcal{I}(N):=\{0,1,2,\cdots, N-1\}$. 
    Given the total amount of resource is $M$, each entry of the action, $\mathbf{a}_i$, corresponds to the action for each resource (resource action). Thus, the action $\mathbf{a}$ is a vector with $M$ entries. 
    Given that each node is assigned an ID from 0 to $N-1$ as defined in Section \ref{subsec:cbg}, each entry in the action vector is interpreted as follows: $``0"$ denotes staying still, $``1"$ represents moving to the node whose ID is 1 greater than the current node's ID, $``2"$ represents moving to the node whose ID is 2 greater than the current node's ID, and so on.  When the target node ID exceeds $N-1$, it wraps around using modulo $N$, ensuring that all node IDs remain within the range $0$ to $N-1$. 
    Therefore, the action space contains $N^M$ entries. 
    
    However, it is important to note that for a GRAG, not all of these $N^M$ entries are valid. Due to the graph constraint, the action space is restricted within a certain range. For example, if node $i$ and node $j$ are not connected, any action involving resource transition from node $i$ to node $j$ would be invalid. Therefore, it is necessary to identify a subset of the action space that contains all valid actions. Moreover, this valid action set depends on the current resource distribution and changes dynamically as the players' resource distributions evolve. Defining the player's action space based on the actions for each resource simplifies the computation of the valid action set, as demonstrated in Section~\ref{subsec:mask}.
    
    \item $P_{ss'}^a=P(S_{t+1}=\mathbf{s}'|\ S_{t}=\mathbf{s}, A_t=\mathbf{a})$ denotes the player's state transition probability at time $t$ when taking action $\mathbf{a}$ at state $\mathbf{s}$.
    \item $R_i(S_{t}=\mathbf{s})$ represents player $i$'s reward at state $\mathbf{s}$. If player $i$ controls more nodes than the other player at state $\mathbf{s}$, then $R_i(S_{t}=\mathbf{s})=1$. If player $i$ controls fewer nodes, $R_i(S_{t}=\mathbf{s})=-1$. If both players control the same number of nodes, $R(S_{t+1}=\textbf{s}')=0$. Formally, we have
        \begin{equation}
            R(S_{t}=\textbf{s})=
                \begin{cases} 
                1,&\!\!\!\!\!\text{if} \sum_i^N\mathbb{I}(\mathbf{s}_i>0) >  \sum_i^N\mathbb{I}(\mathbf{s}_i<0), \\
                0,&\!\!\!\!\!\text{if} \sum_i^N\mathbb{I}(\mathbf{s}_i>0) =  \sum_i^N\mathbb{I}(\mathbf{s}_i<0), \\
                -1,&\!\!\!\!\!\text{if} \sum_i^N\mathbb{I}(\mathbf{s}_i>0) <  \sum_i^N\mathbb{I}(\mathbf{s}_i<0),
                \end{cases}
        \label{eq:reward}
        \end{equation}
    where $\mathbb{I}(\cdot)$ is an indicator such that, $\mathbb{I}(x>0)=1$ if $x>0$ and $\mathbb{I}(x>0)=0$ otherwise. $\sum_i^N\mathbb{I}(\mathbf{s}_i>0) >  \sum_i^N\mathbb{I}(\mathbf{s}_i<0)$ means that there are more positive elements than negative elements in vector $\mathbf{s}$, therefore Player 1 claims the winner and obtains the reward of 1.
\end{itemize}

\subsection{Valid Action Set} 
\label{subsec:mask}
For a GRAG, the range of valid actions for the player at time $t$ is constrained and dynamically changing based on the player's distribution $\mathbf{d}_t$ and the graph $G$ with adjacency matrix $\mathbf{H}$. 
To calculate valid action set at time $t$, we construct \textit{action-displacement adjacency matrix}, denoted as $\mathbf{J}_{t}$. $\mathbf{J}_{t}$ can be computed using the adjacency matrix $\mathbf{H}$ and the resource distribution $\mathbf{d}_t$ at time $t$. For adjacency matrix $\mathbf{H}$, $\mathbf{H}_i$ denotes its $i$-th row and $\mathbf{H}_i^{j}$ denotes the left circular shift of $i$-th row's elements by $j$ times. For example, if $\mathbf{H}_i=\mathbf{H}_i^0=(1,2,3,4)$, $\mathbf{H}_i^1=(2,3,4,1)$ and $\mathbf{H}_i^2=(3,4,1,2)$. Consider that the resource distribution at time $t$ is $\mathbf{d}_t=(m_1, m_2, \cdots, m_N)$ with $\sum_{i=1}^N m_i=M$. Action-displacement adjacency matrix $\mathbf{J}_t$ is formed by concatenating $m_i$ copies of $\mathbf{H}_i^{i-1}$ for $i$ from $1$ to $N$, and thus $\mathbf{J}_t$ is a $M\times N$ dimensional matrix. For example, if $\mathbf{d}_t=(3,0,1,2)$, (meaning total resource amount $M=6$) and the adjacency matrix is
\begin{equation*}
    \mathbf{H}=\begin{pmatrix}
    1&0&1&1\\
    1&1&1&1\\
    1&1&0&1\\
    0&0&1&1
    \end{pmatrix}.
    \label{eq:adjacency}
\end{equation*}
Then the action-displacement adjacency matrix at time $t$ is
\begin{equation}
    \mathbf{J}_t = \begin{pmatrix}
        \vspace{1mm}
        H_1^0\\
        \vspace{1mm}
        H_1^0\\
        \vspace{1mm}
        H_1^0\\
        \vspace{1mm}
        H_3^2\\
        \vspace{1mm}
        H_4^3\\
        \vspace{1mm}
        H_4^3
    \end{pmatrix}
    = \begin{pmatrix}
        1&0&1&1\\
        1&0&1&1\\
        1&0&1&1\\
        0&1&1&1\\
        0&1&1&0\\
        0&1&1&0
    \end{pmatrix}.
    \label{eq:J}
\end{equation}
To simplify notation, we denote $\mathbf{J}_t$ as $\mathbf{J}$ in the subsequent discussion.

First, we assign unique IDs to resources ranging from 0 to $M-1$, ensuring that resources in nodes with smaller IDs receive lower-numbered IDs. For example, for distribution $(3,0,1,2)$, the three resources in node 0 are assigned IDs 0, 1, and 2; the resource in node 2 is assigned ID 3; and the two resources in node 3 are assigned IDs 4 and 5. The row index of $\mathbf{J}$ corresponds to the resource index (from 0 to $M-1$) and the column index of $\mathbf{J}$ corresponds to the resource action (from 0 to $N-1$). Each entry in 
$\mathbf{J}$ indicates whether a particular action is valid. For example, in Eq.~\ref{eq:J}, $\mathbf{J}_{1,2}=1$ signifies that resource $1$ (located on node 0) can take action $2$, while $\mathbf{J}_{2,1}=0$ indicates that resource $2$ (also on node 0) can not take action $1$. A player's valid action $\mathbf{a}$ consists of the valid actions of all resources. Therefore, a valid action set can be formulated as 
\begin{equation}
\label{eq:reachableset1}
\mathcal{A}=\{\mathbf{a}\in\mathbb{R}^M\ |\ \mathbf{J}_{i,\mathbf{a}_i}=1, i\in\mathcal{I}(M)\}.
\end{equation}
For example, action $\mathbf{a}=(0,2,2,3,1,2)$ is a valid action since for all 
$i\in\mathcal{I}(M)$, $\mathbf{J}_{i,\mathbf{a}_i}=1$. However, action $\mathbf{a}=(0,1,2,3,1,2)$ is invalid as $\mathbf{J}_{1,\mathbf{a}_1}=\mathbf{J}_{1,1}=0$. This occurs because, under the distribution $(3,0,1,2)$, selecting $\mathbf{a}_1=1$  implies that the player attempts to move resource 2 (located at node 0) to node 2. However, since $\mathbf{H}_{0,1}=0$, there is no direct path from node 1 to node 2, making this action infeasible.

We formally state that all actions in $\mathcal{A}$ (Eq.~\ref{eq:reachableset1}) are valid actions in Theorem \ref{the:1}.
\begin{theorem} For any resource $j$ on node $n_j$ at time step $t$, $j\in\mathcal{I}(M)$ and $n_j\in\mathcal{I}(N)$, $\mathbf{a}_j\in\mathcal{K}:=\{k\in\mathcal{I}(N)\ |\ \mathbf{J}_{jk}=1\}$ is valid resource action.
\label{the:1}
\end{theorem}

\begin{proof}

We show that for any resource \( j \in \{0, \ldots, M-1\} \), located at node \( n_j \in \{0, \ldots, N-1\} \), a displacement \( a_j \in \mathcal{I}(N) \) is valid if and only if it corresponds to a reachable node according to the graph topology encoded in the adjacency matrix \( \mathbf{H} \in \{0,1\}^{N \times N} \).

We define the left cyclic permutation matrix \( \mathbf{P} \in \{0,1\}^{N \times N} \), whose entries are given by:
\begin{equation}
\mathbf{P}_{ij} = 
\begin{cases}
1, & \text{if } j \equiv i - 1 \pmod{N}, \\
0, & \text{otherwise}.
\end{cases}
\end{equation}
This matrix satisfies \( \mathbf{e}_k \mathbf{P} = \mathbf{e}_{(k+1)\bmod N} \), and right-multiplying a row vector by \( \mathbf{P}^i \) performs a left cyclic shift by \( i \) positions.

Using this operator, we construct the \textit{displacement-augmented adjacency matrix} \( \mathbf{H}' \in \{0,1\}^{N \times N} \) as:
\begin{equation}
\mathbf{H}' = \sum_{i=0}^{N-1} \mathrm{diag}(\mathbf{e}_i) \cdot \mathbf{H} \cdot \mathbf{P}^i.
\end{equation}
This definition ensures that the \( i \)-th row of \( \mathbf{H} \) is cyclically left-shifted by \( i \) positions, and the result placed back into the \( i \)-th row of \( \mathbf{H}' \).

Let \( \mathbf{d} = (d_0, \ldots, d_{N-1}) \in \mathcal{I}(N)^N \) be the node-level resource distribution with \( \sum_i d_i = M \) total resources. We define a \textit{lifting operator} \( \Phi: \mathcal{I}(N)^N \to \{0,1\}^{M \times N} \) which expands \( \mathbf{d} \) into a row-wise stack of one-hot vectors. In particular, if resource \( j \) is located at node \( n_j \), then the \( j \)-th row of \( \Phi(\mathbf{d}) \) equals the canonical basis vector \( e_{n_j} \in \mathbb{R}^N \). With a lifting operator, we have a matrix form resource distribution $\hat{\mathbf{d}}=\Phi(\mathbf{d})$. 
Based on that, we define the valid action matrix as:
\begin{equation}
\mathbf{J} := \hat{\mathbf{d}} \cdot \mathbf{H}'.
\end{equation}
To isolate the valid actions for resource \( j \), we compute the \( j \)-th row of \( \mathbf{J} \) as:
\begin{align*}
\mathbf{J}_{j,\cdot} &= \mathbf{e}_{n_j}^\top \cdot \mathbf{H}' \\
&= \sum_{i=0}^{N-1} \mathbf{e}_{n_j}^\top \cdot \mathrm{diag}(\mathbf{e}_i) \cdot \mathbf{H} \cdot \mathbf{P}^i \\
&= \mathbf{e}_{n_j}^\top \cdot \mathrm{diag}(\mathbf{e}_{n_j}) \cdot \mathbf{H} \cdot \mathbf{P}^{n_j} \\
&= H_{n_j,\cdot} \cdot \mathbf{P}^{n_j}.
\end{align*}
Here, the final expression represents the \( n_j \)-th row of \( \mathbf{H} \) left-shifted by \( n_j \) positions. Therefore, the \( k \)-th element of this row equals:
\begin{equation}
\mathbf{J}_{j,k} = H_{n_j, (n_j + k) \bmod N}.
\end{equation}
Thus, \( \mathbf{J}_{j,k} = 1 \) if and only if the original adjacency matrix \( \mathbf{H} \) allows a transition from node \( n_j \) to node \( (n_j + k) \bmod N \). This completes the proof.
\end{proof}

Since all the proof processes are reversible, conversely, all valid actions $k$ satisfy $\mathbf{J}_{jk}=1$ for resource $j$. 
In Figure~\ref{fig:illus_2}, calculation \textcircled{1} illustrates the process of generating the valid action space from the action-displacement adjacency matrix (Eq.~\ref{eq:reachableset1}).

\textit{Theorem}~\ref{the:1} states that for each resource $j$, its valid action set consists of the column indices of all the entries of ``1'' in the corresponding $j$-th row of $\mathbf{J}$. 
Each resource selects a valid action, and the combination of these actions constitutes a valid action $\mathbf{a}=(a_0,a_1,\cdots, a_{M-1})$ for a player. 
Therefore, the valid action set is $\mathcal{A}$ in Eq.~\ref{eq:reachableset1}.

\begin{theorem}
    If a player resource distribution at time $t$ is $\mathbf{d}$, and she takes action $\mathbf{a}_t=(a_0,\cdots,a_{M-1})$ and ends up with $\mathbf{d}'$, then
    \begin{equation}
        \mathbf{d}' = \sum_{j=0}^{M-1} \hat{\mathbf{d}}_{j,:} \cdot \mathbf{P}^{a_j}.
        \label{eq:the2}
    \end{equation}
    \label{the:2}
\end{theorem}  

    
\begin{proof}
Let \( \hat{\mathbf{d}} = \Phi(\mathbf{d}) \in \{0,1\}^{M \times N} \) be the resource-level position matrix at time \( t \), where each row is a one-hot vector indicating the current location of resource \( j \).

For each resource \( j \), its location vector \( \hat{\mathbf{d}}_{j,:} \) is shifted left by \( a_j \) positions via right multiplication with \( \mathbf{P}^{a_j} \), resulting in a new one-hot vector indicating its new location. Aggregating the shifted vectors of all resources yields the updated node-level distribution:
\begin{equation*}
\mathbf{d}' = \sum_{j=0}^{M-1} \hat{\mathbf{d}}_{j,:} \cdot \mathbf{P}^{a_j}.
\end{equation*}
This completes the proof.
\end{proof}

In Figure~\ref{fig:illus_2}, calculation \textcircled{2} represents Eq.~\ref{eq:the2}. At this point, we determine the valid action set $\mathcal{A}_t$ based on the current distribution $\mathbf{d}_t$, as well as how to compute the distribution at the next time step $\mathbf{d}_{t+1}$ based on the action $\mathbf{a}_t$ and the current distribution $\mathbf{d}_t$.
\begin{figure*}
    \centering
    \includegraphics[width=0.95\linewidth]{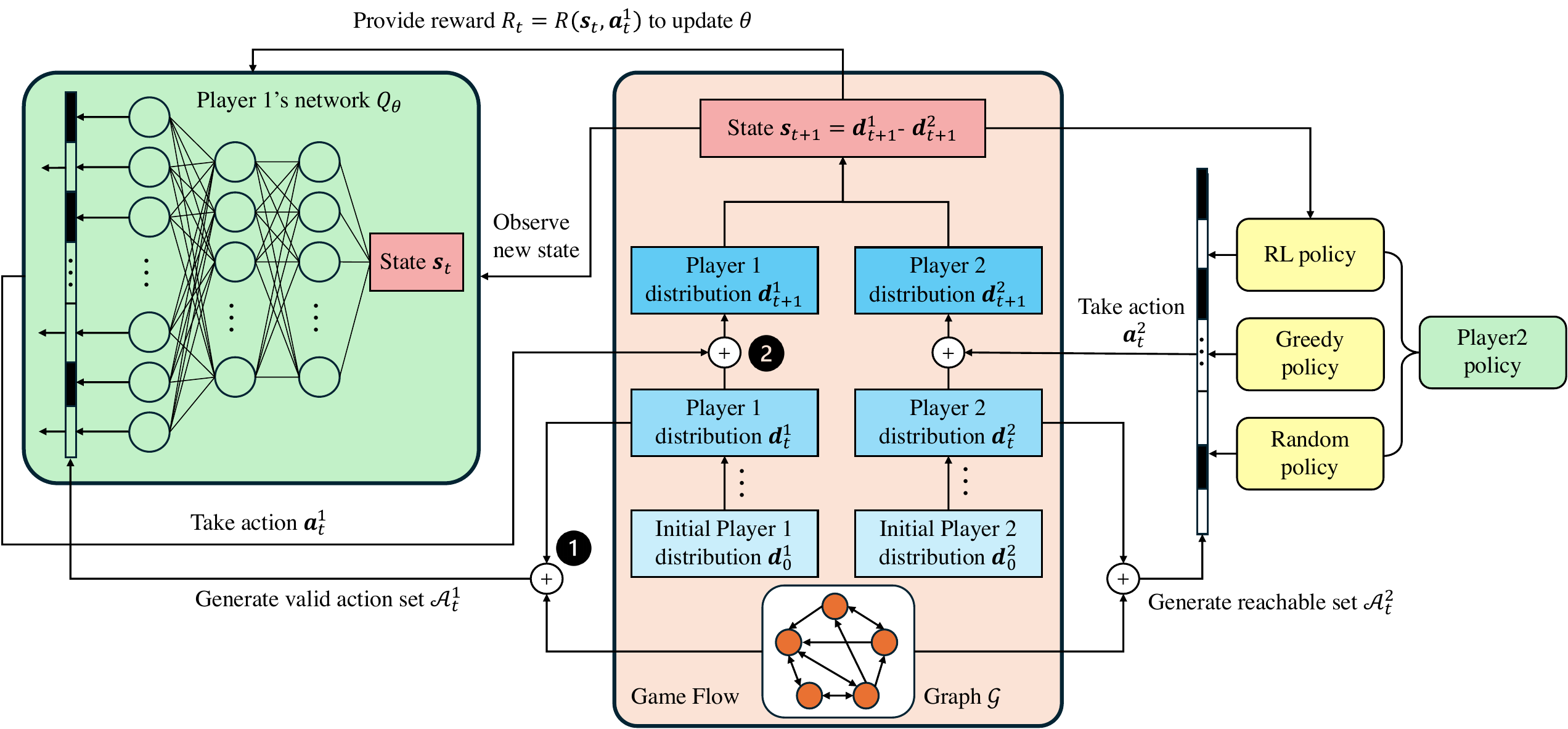}
    \caption{Game flow and DQN process with Player 1 as the RL agent.}
    \label{fig:illus_2}
\end{figure*}
\subsection{Action-displacement adjacency matrix with different resources}
\label{subsec:diff-res}
In this study, we primarily focus on scenarios where both players have the same total amount of resources. However, to further demonstrate the capability of RL in handling  GRAG, we will explore cases with unequal total resources in subsequent experiments.
Most RL APIs require that the action space dimensions of both players be consistent during setup. When the total resources of the two players differ, this results in differing action space dimensions ($N^{M_1}\neq N^{M_2}$). Below, we first unify the action space dimensions for scenarios with unequal total resources. Based on this unified action space, we will then describe how to generate the valid action set and compute the distribution at the next time step from the action and current distribution. Due to symmetry, we assume $M_1<M_2$ in the following discussion and denote $M_2=M$ for simplicity.

With both players' action space dimensions are unified as $N^M$, we define Player 1's action $\mathbf{a}_1=(a_0,\cdots,a_{M-1})$ as an $M$-dimensional vector. The first $M-M_1$ components of $\mathbf{a}_1$ corresponding to the actions for the added virtual resources, and these components are constrained to take the value 0. The action-adjacency matrix of Player 1, $\mathbf{J}_1$, becomes 
\begin{equation}
    \mathbf{J}'_1 = \begin{pmatrix}
        \vspace{0.5mm}
        \mathbf{1}^T_{M-M_1},\mathbf{0}\\
        \vspace{0.5mm}
        \mathbf{J}_1
    \end{pmatrix},
\end{equation}
and the valid action set of Player 1 is generated from $\mathbf{J}'_1$ similar in Section~\ref{subsec:mask}. As for the relationship of $\mathbf{d}_t$, $\mathbf{d}_{t+1}$ and $\mathbf{a}_t$, we have \textit{Corollary}~\ref{cor1} based on \textit{Theorem}~\ref{the:2},
\begin{corollary}
    If a player resource distribution at time $t$ is $\mathbf{d}$, and she takes action $\mathbf{a}_t=(a_0,\cdots,a_{M-1})$ with first $M-M_1$ components being actions of virtual resources, and ends up with $\mathbf{d}'$,
    \begin{equation}
        \mathbf{d}' = \sum_{j=0}^{M_1-1} \hat{\mathbf{d}}_{j,:}\cdot\mathbf{P}^{a_{j+M-M_1}}.
        \label{eq:cor1}
    \end{equation}

    \label{cor1}
\end{corollary}

\subsection{DQN Framework}
\label{subsec:dqn}
We outline the steps of using DQN for the  GRAG in Figure~\ref{fig:illus_2}. The overall process begins at the bottom of the middle portion of Figure~\ref{fig:illus_2}, where the initial resource distributions of the two players, \(\mathbf{d}_0^1\) and \(\mathbf{d}_0^2\), are established. Player 1 is modeled as an RL agent, whereas Player 2 can follow a random policy, a fixed heuristic policy, or an RL-based policy, as described in Section~\ref{subsec:greedy}.

At each time step \(t\), the valid action sets \(\mathcal{A}_t^1\) and \(\mathcal{A}_t^2\) are computed for both players based on their current resource configurations and the adjacency matrix \(\mathbf{H}\), using Eq.~\ref{eq:reachableset1}. The current state is defined as the difference between the two players' distributions, \(\mathbf{s}_t = \mathbf{d}_t^1 - \mathbf{d}_t^2\), and the reward signal \(R_t\) for Player 1 is determined by Eq.~\ref{eq:reward}.

Player 1 selects an action \(\mathbf{a}_{t+1}^1\) from its valid action set \(\mathcal{A}_t^1\) using the $\epsilon$-greedy strategy \cite{tokic2010adaptive}: with probability \(\epsilon\), a random valid action is chosen to encourage exploration, and with probability \(1 - \epsilon\), the action with the highest Q-value is selected based on the current Q-network $Q_\theta$. Player 2's action \(\mathbf{a}_{t+1}^2\) is determined by its predefined policy.

The episode terminates when one player wins the game, which is indicated by a non-zero reward \(R_t \neq 0\). During training, the tuple \((\mathbf{s}_t, \mathbf{a}_t, \mathbf{s}_{t+1}, R_t)\) is stored in a replay buffer $\mathcal{B}$. The Q-network $Q_\theta$ is updated periodically by sampling minibatches from the buffer and minimizing the temporal difference (TD) loss, typically defined as:
\[
\mathcal{L}(\theta) = \mathbb{E}_{(s_t, a_t, r_t, s_{t+1}) \sim \mathcal{B}} \left[ \left( Q_\theta(s_t, a_t) - y_t \right)^2 \right],
\]
where \(y_t = R_t + \gamma \max_{a'} Q_{\theta^-}(s_{t+1}, a')\) is the TD target computed using a target network \(Q_{\theta^-}\), which is updated less frequently to stabilize training. 

If Player 2 is also an RL agent, it generates actions using its own Q-network, and its reward is defined symmetrically as \(R_t^2 = -R_t^1\), reflecting the zero-sum nature of the game.

\begin{figure*}
    \centering
    \includegraphics[width=0.95\linewidth]{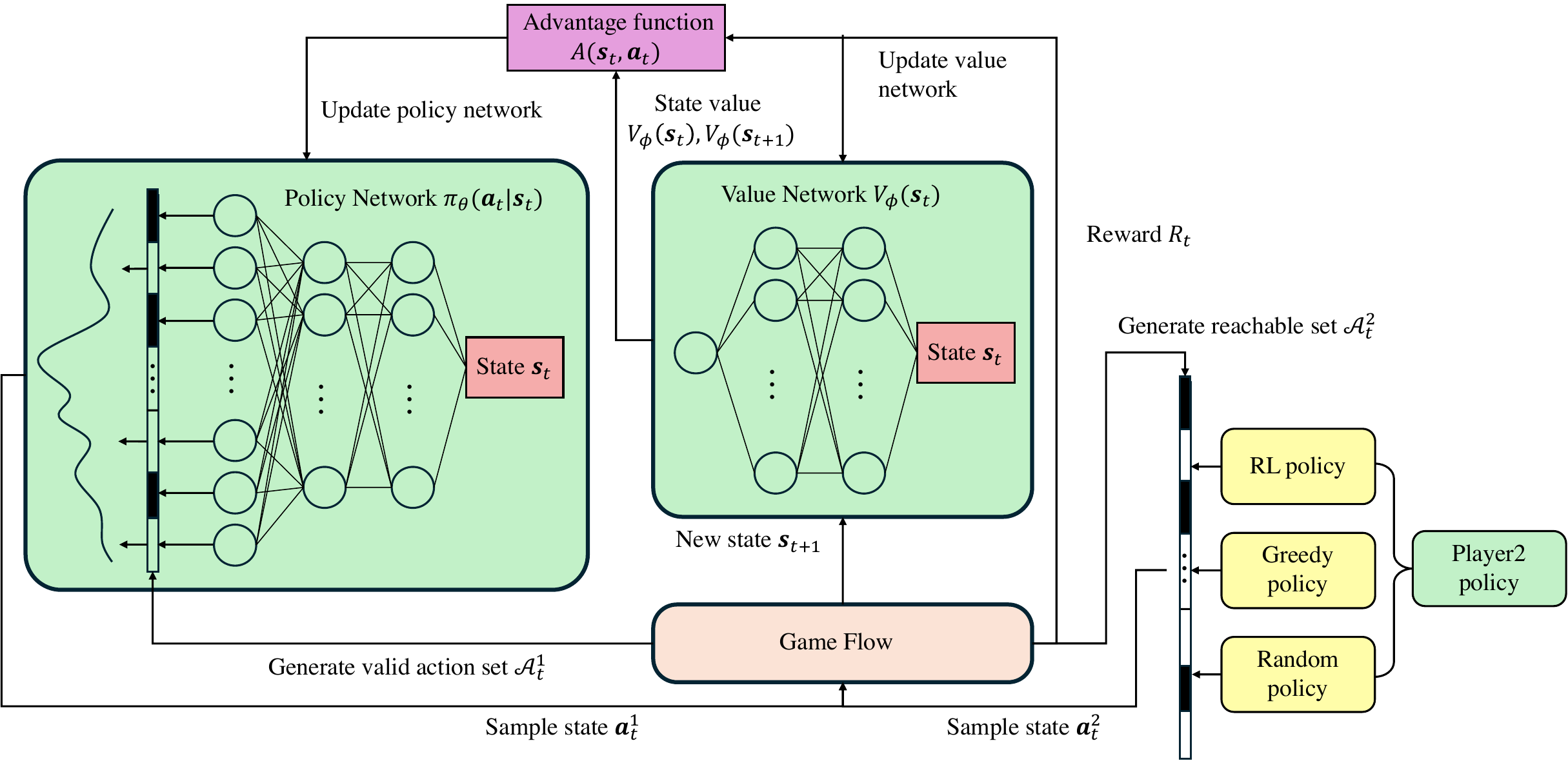}
    \caption{Game flow and PPO network with Player 1 as the RL agent. The game flow structure is the same as in Figure~\ref{fig:illus_2}.}
    \label{fig:illus_3}
\end{figure*}
\subsection{PPO Framework}
\label{subsec:ppo}
The PPO framework shares the same game flow as the DQN framework, as illustrated in Figure~\ref{fig:illus_3}. The difference is that PPO adopts an actor-critic structure with two separate deep neural networks: the policy network \(\pi_\theta\) (left), and the value network \(V_\phi\) (right). At each time step, the policy network serves as the actor, generating actions by sampling from the distribution \(\mathbf{a}_t \sim \pi_\theta(\mathbf{a}|\mathbf{s}_t)\), instead of selecting the action with the highest Q-value as in DQN. The value network estimates the expected return from state \(\mathbf{s}_t\) and is updated based on the observed reward signal \(R_t\).

During training, PPO collects entire episodes and uses them to compute the advantage estimates for each state-action pair, typically via the Generalized Advantage Estimation (GAE) method. The policy network $\pi_\theta$ is then updated using a clipped surrogate objective that restricts large policy updates, thereby ensuring training stability. Specifically, the objective includes a term that penalizes deviations from the previous policy beyond a specified trust region, which helps to prevent destructive updates.

In contrast to DQN, which uses an experience replay buffer and off-policy learning, PPO performs on-policy updates and trains the networks using fresh batches of collected trajectories. Moreover, while DQN updates the policy every few steps based on individual transitions, PPO performs batch updates after collecting multiple steps of interaction.

\section{Experiments}
We conduct experiments to evaluate the effectiveness of DQN and PPO in computing optimal policies for the GRAG across various graph structures, including $G_0, G_1, G_2, G_3$ and $G_4$ shown in Figure~\ref{fig:bias}. 
In our experiments, Player 1 is modeled as an RL agent, while Player 2 follows one of three strategies: a random policy, a greedy policy, or an RL-based policy.
\begin{figure*}
    \centering
    \subfloat[$G_0$]{
    \includegraphics[width=0.2\linewidth]{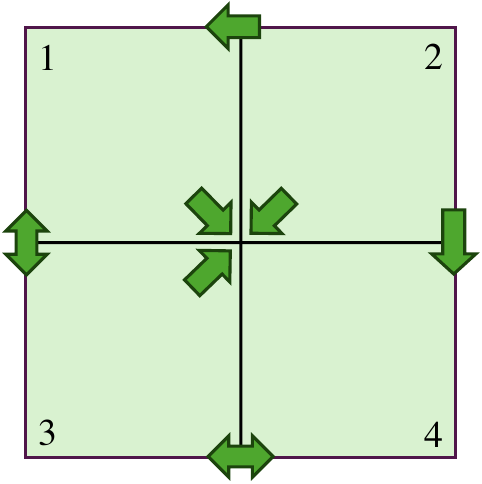}}
    \subfloat[$G_1$]{
    \includegraphics[width=0.21\linewidth]{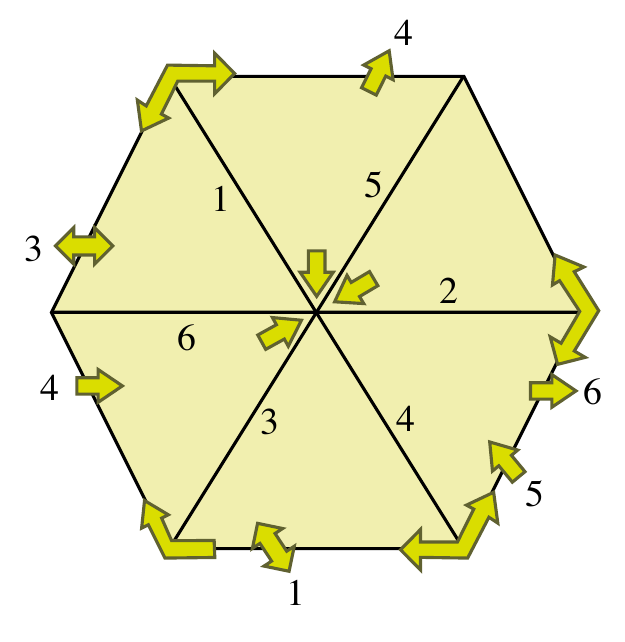}}
    \subfloat[$G_2$]{
    \includegraphics[width=0.2\linewidth]{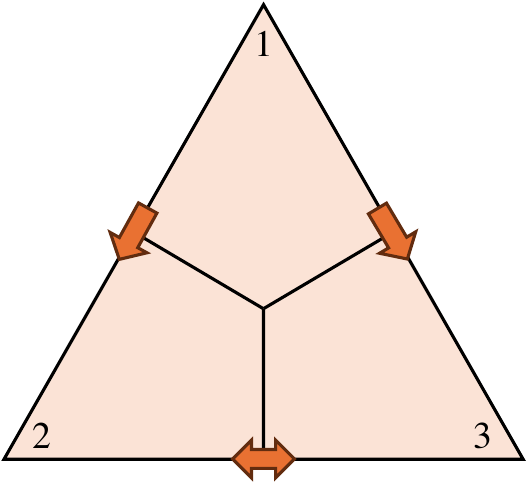}}
    \subfloat[$G_3$]{
    \includegraphics[width=0.2\linewidth]{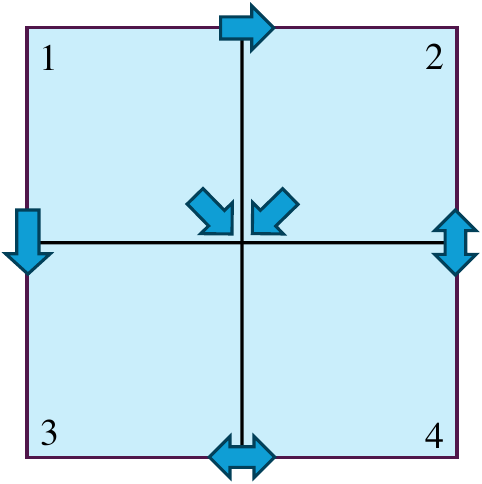}}
    \subfloat[$G_4$]{
    \includegraphics[width=0.19\linewidth]{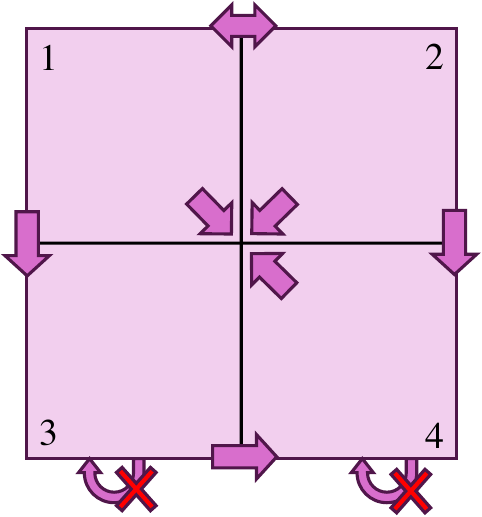}}
    \caption{Five graphs $G_0, G_1, G_2, G_3, G_4$, where a node is indexed numerically and represents a discrete location for resource allocation. Directed arrows between nodes indicate one-way connectivity, specifying the allowable transitions of resources between locations. }
    \label{fig:bias}
\end{figure*}

\begin{figure*}
    \centering
    \subfloat[$C_1, C_2, C_3, C_4$ (from left to right) on $G_0$]{
    \label{g0}
    \includegraphics[width=0.85\linewidth]{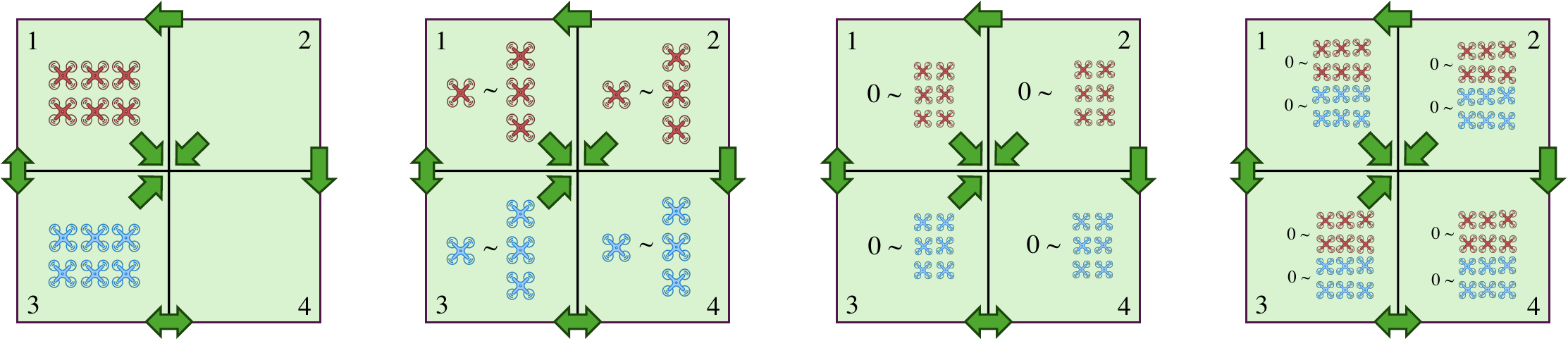}}

    \subfloat[$C_1, C_2, C_3, C_4$ (from left to right) on $G_1$]{
    \label{g1}
    \includegraphics[width=0.85\linewidth]{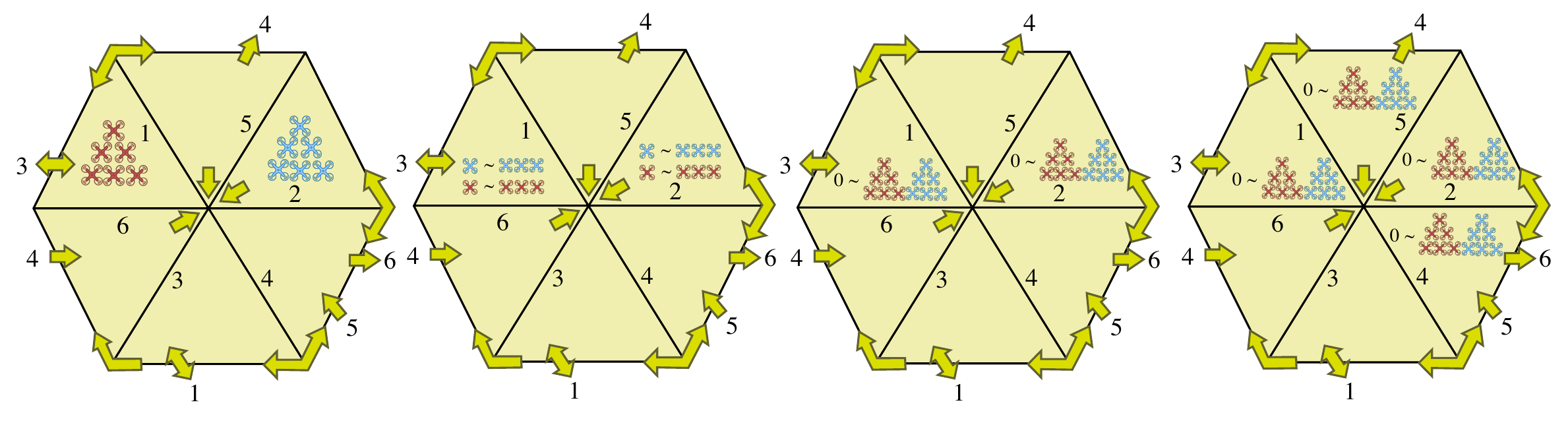}}
    \caption{Illustration of four distinct initial resource distributions $C_1$, $C_2$, $C_3$ and $C_4$ on graphs $G_0$ (\ref{g0}) and $G_1$ (\ref{g1}).}
    \label{fig:initial}
\end{figure*}
These five cases can be classified into two categories based on the prescribed initial resource distribution. $G_0$ and $G_1$ represent symmetric graphs, where the structural characteristics do not provide any inherent advantage to either Player 1 or Player 2, thereby constituting a fair game. Conversely, $G_2$, $G_3$, and $G_4$ are asymmetric graphs, where the initial resource distributions are deliberately crafted to impose a significant disadvantage on one player, while granting a corresponding advantage to the other, resulting in an inherently unfair competition. We assess the learning performance by comparing the win rates of both players under fair and unfair conditions. Specifically,
on $G_0$ and $G_1$, we consider four types of initializations of resource distribution $C_1\sim C_4$, from fixed to random. 
The four types for graphs $G_0$ and $G_1$ are illustrated in Figure~\ref{g0} and Figure~\ref{g1}, respectively.
For graph $G_0$, $C_1$ corresponds to a deterministic distribution where Player 1 allocates all resources to node 1, and Player 2 allocates all resources to node 3. In $C_2$, Player 1 allocates a random resource units from 1 to 3 to node 1, with the rest to node 2; Player 2 does the same to nodes 3 and 4. In $C_3$, Player 1 freely distributes resources to nodes 1 and 2, and Player 2 freely distributes resources to nodes 3 and 4. In $C_4$, both players can allocate resources freely across all nodes.
For graph $G_1$, $C_1$ is deterministic, with Player 1 allocating all resources to node 1, and Player 2 allocating all resources to node 2. In $C_2$, both players randomly allocate from 1 to 3 resource units to nodes 1 and 2. In $C_3$, both players freely distribute resources between nodes 1 and 2. In $C_4$, all resources can be freely allocated across nodes 1, 2, 4, and 5.
The prescribed initializations for unfair games on $G_2$, $G_3$, and $G_4$ will be elaborated in Section~\ref{subsec:rl-agent-g234}.


We build our environment using \textit{PettingZoo} \cite{terry2021pettingzoo}. \textit{PettingZoo} is a Python library for multi-agent reinforcement learning (MARL). It provides a standardized interface for various multi-agent environments. We utilize \textit{Tianshou} \cite{tianshou} for training.
\textit{Tianshou} is an RL library designed for efficient and flexible experimentation. Built on PyTorch, it provides a wide range of state-of-the-art RL algorithms and tools to simplify the development and evaluation of RL policies. 
    
\begin{table*}
\caption{Win rate of DQN against distinct policies with four distinct distribution initialization on $G_0$.}
    \centering
    \begin{tabular}{lrrrrrrrr}
        \toprule
        Initialization & \multicolumn{2}{l}{$C_1$} & \multicolumn{2}{l}{$C_2$} & \multicolumn{2}{l}{$C_3$} & \multicolumn{2}{l}{$C_4$} \\
        \midrule
        \diagbox{Train}{Test}&Trained&Random&Trained&Random&Train&Random&Trained&Random \\ 
        \midrule
        Random    & \multicolumn{2}{c}{80\%}     &\multicolumn{2}{c}{74\%}      &\multicolumn{2}{c}{67\%}     &\multicolumn{2}{c}{64\%}           \\
        Greedy $\pi_0$        & 100\%    &$\sim$50\%        &92\%&60\% &83\%&58\%& 71\%&62\%    \\
        Greedy $\pi_1$        & 100\%    &$\sim$50\%        &89\%&55\% &81\%&56\%& 69\%&63\%    \\
        Greedy $\pi_2$        & 100\%    &$\sim$50\%        &90\%&53\% &84\%&59\%& 65\%&59\%    \\
        RL                    & --       & --               &$\sim$50\%& 56\%&$\sim$50\%& 54\%& $\sim$50\%&53\% \\
        \bottomrule
    \end{tabular}
    \label{tab:DQN-1}
\end{table*}

\begin{table*}
\caption{Win rate of PPO against distinct policies with four distinct distribution initialization on $G_0$.}
    \centering
    \begin{tabular}{lrrrrrrrr}
        \toprule
        Initialization & \multicolumn{2}{l}{$C_1$} & \multicolumn{2}{l}{$C_2$} & \multicolumn{2}{l}{$C_3$} & \multicolumn{2}{l}{$C_4$}\\
        \midrule
        \diagbox{Train}{Test}&Trained&Random&Trained&Random&Train&Random&Trained&Random \\ 
        \midrule
        Random    & \multicolumn{2}{c}{68\%}     &\multicolumn{2}{c}{66\%}      &\multicolumn{2}{c}{59\%}     &\multicolumn{2}{c}{60\%}           \\
        Greedy $\pi_0$        & 100\%    &$\sim$50\%         &93\%&58\% &74\%&50\%& 64\%&53\%    \\
        Greedy $\pi_1$        & 100\%    &$\sim$50\%         &95\%&55\% &60\%&43\%& 62\%&53\%    \\
        Greedy $\pi_2$        & 100\%    &$\sim$50\%         &92\%&60\% &65\%&51\%& 63\%&45\%    \\
        RL                    & --       & --                &$\sim$50\%&37\% &$\sim$50 & 57\%& $\sim$50\%&44\% \\
        \bottomrule
    \end{tabular}
    \label{tab:PPO-1}
\end{table*}

\begin{table*}
\caption{Win rate of DQN against distinct policies with four distinct distribution initialization on $G_1$.}
    \centering
    \begin{tabular}{lrrrrrrrr}
        \toprule
        Initialization & \multicolumn{2}{l}{$C_1$} & \multicolumn{2}{l}{$C_2$} & \multicolumn{2}{l}{$C_3$} & \multicolumn{2}{l}{$C_4$} \\
        \midrule
        \diagbox{Train}{Test}&Trained&Random&Trained&Random&Train&Random&Trained&Random \\ 
        \midrule
        Random    & \multicolumn{2}{c}{77\%}     &\multicolumn{2}{c}{74\%}      &\multicolumn{2}{c}{65\%}     &\multicolumn{2}{c}{64\%}           \\
        Greedy $\pi_0$        & 100\%    &$\sim$50\%        &95\%&57\% & 78\%&60\% &83\%&62\%   \\
        Greedy $\pi_1$        & 100\%    &$\sim$50\%        &91\%&55\% & 80\%&58\% &81\%&61\%   \\
        Greedy $\pi_2$        & 100\%    &$\sim$50\%        &88\%&57\% & 77\%&61\% &84\%&63\%   \\
        RL                    & --       & --               &$\sim$50\%& 56\%&$\sim$50\% &45\% &$\sim$50& 39\%\\
        \bottomrule
    \end{tabular}
    \label{tab:DQN-2}
\end{table*}

\begin{table*}
\caption{Win rate of PPO against distinct policies with four distinct distribution initialization on $G_1$.}
    \centering
    \begin{tabular}{lrrrrrrrr}
        \toprule
        Initialization & \multicolumn{2}{l}{$C_1$} & \multicolumn{2}{l}{$C_2$} & \multicolumn{2}{l}{$C_3$} & \multicolumn{2}{l}{$C_4$}\\
        \midrule
        \diagbox{Train}{Test}&Trained&Random&Trained&Random&Train&Random&Trained&Random \\ 
        \midrule
        Random    & \multicolumn{2}{c}{65\%}     &\multicolumn{2}{c}{61\%}      &\multicolumn{2}{c}{59\%}     &\multicolumn{2}{c}{55\%}           \\
        Greedy $\pi_0$        & 100\%    &$\sim$50\%         &85\%&55\% & 78\%&51\% &74\%&50\%   \\
        Greedy $\pi_1$        & 100\%    &$\sim$50\%         &85\%&60\% & 74\%&53\% &60\%&43\%   \\
        Greedy $\pi_2$        & 100\%    &$\sim$50\%         &87\%&54\% & 75\%&58\% &65\%&51\%   \\
        RL                    & --       & --                &$\sim$50\%&-\% &$\sim$50\% &55\% &$\sim$50 & 40\%\\
        \bottomrule
    \end{tabular}
    \label{tab:PPO-2}
\end{table*}
\subsection{RL agent against random-policy agent on $G_0$ and $G_1$}
We set Player 1 to be the RL agent and Player 2 to be a random-policy agent.
In the case of fixed initialization $C_1$, the DQN agent can achieve an average of $77\%~80\%$ win rate, and the PPO agent can achieve an average of $65\%~68\%$ win rate. Both RL agents consistently outperform a random-policy agent, highlighting their effectiveness. 
When the initial distribution is randomized, the win rates decrease, as shown in the first row of Table~\ref{tab:DQN-1} $\sim$ \ref{tab:PPO-2}. However, both DQN and PPO still maintain a higher win rate compared to the random-policy agent, demonstrating their effectiveness for solving  GRAG on graphs.
\subsection{RL agent against greedy-policy agent on $G_0$ and $G_1$}
\label{subsec:greedy}
We further evaluate the RL agent’s performance against a greedy-policy agent through two types of experiments. In the first test, Player 1 is an RL agent trained against a greedy-policy agent, while Player 2 follows a greedy-policy strategy. In the second test, Player 1 remains the RL agent trained against the greedy-policy agent, but Player 2 adopts a random policy to assess the model's generalization capability.
The greedy policy is based on a pre-trained DQN model, where at each step, the agent selects the action that maximizes the pre-trained $Q'(\mathbf{s}, \mathbf{a})$ value, defined as:
\begin{equation*}
    \mathbf{a}_t=\arg \max_{\mathbf{a}} Q'(\mathbf{s}_t,\mathbf{a}).
    \label{eq:greedy}
\end{equation*}
We now describe the process of obtaining the pre-trained DQN model.

First, we train a DQN model against a random-policy agent, and obtain $Q_0(\mathbf{s}, \mathbf{a})$. Using this model, we derive a greedy policy: $\pi_0(\mathbf{a}|\mathbf{s}) = \arg \max_{\mathbf{a}} Q_0(\mathbf{s},\mathbf{a})$. Next, we set Player 2 to follow the greedy policy $\pi_0$ and train Player 1 (the RL agent) against it, resulting in an updated policy $Q_1(\mathbf{s}, \mathbf{a})$.
This iterative process continues, where Player 1 is repeatedly trained against the evolving greedy policy $\pi_i(\mathbf{a}|\mathbf{s})$, yielding a new DQN model $Q_{i+1}(\mathbf{s},\mathbf{a})$ at each iteration.

\textbf{DQN}: In case of deterministic distribution initialization $C_1$, we first obtain a series of pre-trained DQNs $Q_0, Q_1, Q_2$, and corresponding greedy policies $\pi_0, \pi_1, \pi_2$. Take $\pi_0$ for example, in the test stage, Player 1's win rate against greedy policy $\pi_0(\mathbf{a}|\mathbf{s})$ is $100\%$, as shown in the first column of Table~\ref{tab:DQN-1} and Table~\ref{tab:DQN-2}. This result occurs because, under a deterministic initial resource distribution $\mathbf{d}^1_0$ and $\mathbf{d}_0^2$, the initial state $\mathbf{s}_0=\mathbf{d}^1_0-\mathbf{d}_0^2$ remains constant. Therefore, the pre-trained DQN, $Q_0$ yields a predefined greedy action under deterministic initialization, meaning $\pi_0(\mathbf{a}|\mathbf{s}_0)=\arg\max_\mathbf{a}Q_0(\mathbf{s}_0,\mathbf{a})$ is a constant vector. Consequently, the pre-trained DQN $Q_0$ determines a deterministic greedy action under this specific initialization, meaning that $\pi_0(\mathbf{a}|\mathbf{s}_0)=\arg\max_{\mathbf{a}}Q_0(\mathbf{s}_0,\mathbf{a})$ is always a constant vector. As a result, the RL agent only needs to learn how to counter a single deterministic strategy based on $\mathbf{d}_0^2$, making it highly effective in this specific setting. However, when Player 1 competes against a random-policy agent during testing, its win rate drops to approximately $50\%$, indicating a significant struggle to generalize to more complex and diverse scenarios. The performance of the remaining greedy policies follows a similar pattern to $\pi_0$. We hypothesize that this lack of generalization arises due to the overly specific and narrow training conditions, which fail to expose the RL agent to a broader range of strategic variations.

For random distribution initializations $C_2$, $C_3$, and $C_4$, the results show that Player 1 maintains a relatively high win rate against the trained (greedy) agent and random-policy agent. This suggests that the DQN agent trained under random initialization exhibits a certain degree of generalization. In addition, as the level of initialization randomness increases, this generalization capability further improves, as shown in the second to the fourth rows of Table~\ref{tab:DQN-1} and Table~\ref{tab:DQN-2}.
Furthermore, when tested against various greedy policies ($\pi_0$, $\pi_1$, $\pi_2$), the DQN agent consistently achieves high win rates. It is also observed that, as the level of initialization randomness increases, the win rate against the trained (greedy) policies tends to decline. These results indicate that the level of initialization randomness plays a trade-off between the performance and generalization of the DQN agent.


\textbf{PPO}: Similarly, for random distribution initialization, we train PPO against various greedy policies $\pi_0,\pi_1,\pi_2$.
Unlike DQN, PPO training cannot be structured as a sequential process where each policy is trained against the previous one. Instead, PPO is trained independently against each fixed policy $\pi_i$. We evaluate whether PPO exhibits generalization after being trained on different fixed policies. The results are shown in Table~\ref{tab:PPO-1} and Table~\ref{tab:PPO-2}. The table reveals that PPO performs less effectively compared to DQN. The training results for PPO vary significantly across different greedy policies, and the PPO models trained against specific greedy policies struggle to perform well against the random policy, with win rates fluctuating around $50\%$.

We conjecture two key factors that may result in this outcome. First, the PPO, as an on-policy method, has lower sample efficiency, meaning it requires more training time and a larger dataset to achieve competitive performance. In contrast, DQN benefits from its off-policy nature, which allows for more efficient sample reuse, enabling it to converge faster. Second, the fundamental difference between value-based and policy-based learning plays a crucial role. DQN is inherently well-suited for GRAG, where the action space is discrete, making value-based methods more efficient. On the other hand, PPO is typically better at handling continuous action spaces, which may explain its comparatively weaker performance in this setting.


\subsection{RL agent against RL agent on $G_0$ and $G_1$}
When both players are RL agents, they update their policies based on each other's. 
In our setup, the state $\mathbf{s}$ is defined as the difference in resource distributions between the two players, i.e., $\mathbf{s}_t=\mathbf{d}_t^1-\mathbf{d}_t^2$. This value is identical for both players, requiring both agents to determine the optimal next action based on the same shared state. Each agent then takes an action $\mathbf{a}_t^1, \mathbf{a}_t^2$, updates the state $\mathbf{s}_{t+1}$, and receives rewards that are equal in magnitude but opposite in sign. Each agent uses these rewards to update their respective DNNs. This process repeats until the maximum number of iterations is reached.

We conducted multiple training experiments. The results vary across runs: sometimes, Player 1 achieves a higher win rate, while other times, Player 2 does. However, the win rates never significantly favor one side over the other, and the averaged win rates for either player approach $50\%$.
We attribute these fluctuations to the limited number of training iterations. If the training were sufficiently extensive, and the graph structure and initial resource distributions were symmetric, the win rates would likely converge to $50\%-50\%$. Conversely, if the graph or the initial resource distribution is asymmetric, the win rates might deviate from $50\%$. In such cases, the win rate imbalance reflects an inherent advantage for one player due to the asymmetry.

Additionally, we test the generalization performance of the models obtained through mutual training when faced with a random-policy agent. The DQN agent can achieve $62\%$ win rate, while the PPO agent only obtains $\sim 50\%$ win rate, as shown in Table~\ref{tab:DQN-1} $\sim$ \ref{tab:PPO-2}. This further indicates that DQN has better generalization capabilities.
\subsection{RL agent against random-policy agent on $G_0$ with different amount of resources}
We evaluate the performance of DQN and PPO in scenarios where the two players have unequal resources on graph $G_0$.
\begin{figure}
    \centering
    \subfloat[]{
    \label{g0_1}
    \includegraphics[width=0.8\linewidth]{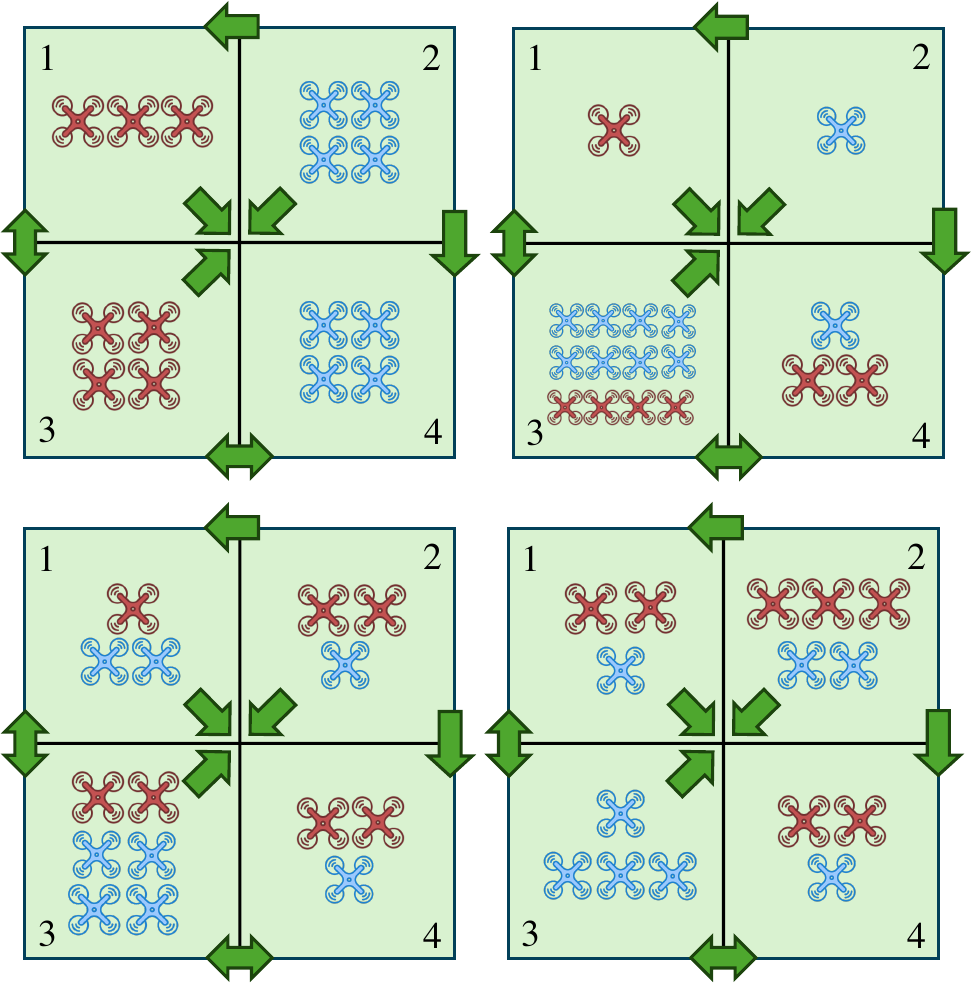}
    }
    
    \subfloat[]{
    \label{g0_2}
    \includegraphics[width=0.8\linewidth]{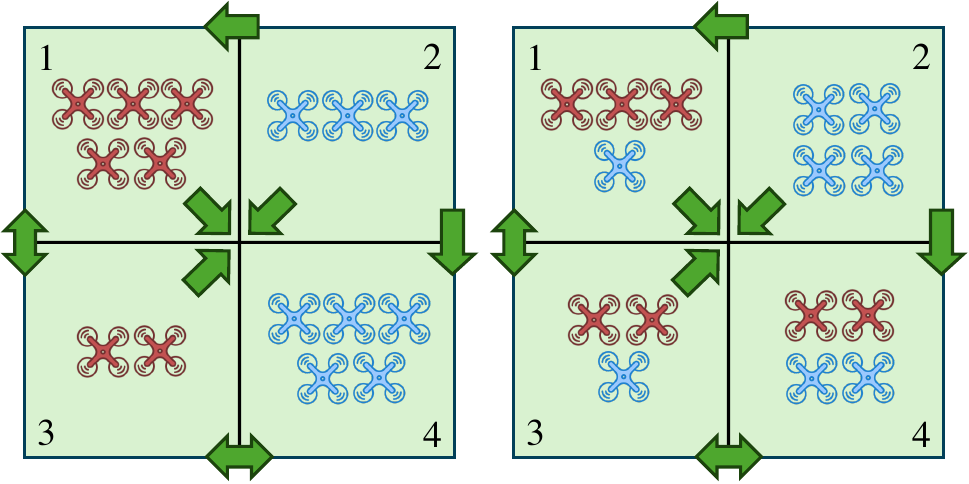}
    }
    \caption{Examples of gameplay on $G_0$ where Player 1 (red) has seven resources and Player 2 (blue) has eight resources.}
    \label{fig:m78}
\end{figure}
In our experiments, the RL agent (Player 1) is assigned 7 resources, while Player 2 has 8 resources. We use the case where both players follow a random policy as a baseline for comparison. Since win rates under random policies are primarily determined by resource quantities, we first examine the performance of Player 1 in this baseline setting. When both players follow random policies on the symmetric graph $G_0$, Player 1 achieves only a $17\%$ win rate due to its resource disadvantage. However, after applying DQN policy to Player 1, Player 1's win rate improves to $24\%$, and with PPO, it reaches $23\%$. These results demonstrate that DQN and PPO significantly enhance performance, even when Player 1 has fewer resources than its opponent.

In Figure~\ref{fig:m78} we present two example episodes in which the RL-trained Player 1 (red) competes against a random-policy Player 2 (blue) on the graph. 
In Fig~\ref{g0_1}, Player 1 starts on nodes 2 and 3, and Player 2 starts on nodes 1 and 4. As the episode progresses, Player 1 attempts to reinforce node 4, but loses control of node 3. As Player 2 allocates more resources to node 3, Player 1 gives up node 3 and wins by controlling all the other nodes.
In Fig~\ref{g0_2}, they begin with a different random initial configuration. Here, Player 1 starts on nodes 1 and 4. Leveraging its trained policy, Player 1 reinforces these positions effectively, while Player 2's random actions result in less coordinated movement. As a result, Player 1 gains control over a majority of the nodes and secures a win.
These examples highlight that, although Player~1 has fewer resources, its learned strategy can exploit certain favorable initial states to outperform a random opponent. This demonstrates the effectiveness of RL in generating competitive allocation strategies even with fewer resources.

\subsection{RL agent against random-policy agent on $G_2, G_3$ and $G_4$}
\label{subsec:rl-agent-g234}
\begin{table}
\caption{Performance of DQN agents and PPO agents trained on biased graphs $G_2,G_3,G_4$.} 
    \centering
    \begin{tabular}{lrrrr}
        \toprule
                            &$G_0$& $G_2$   & $G_3$     &$G_4$\\
        \midrule
        DQN on Player 1      &65\%& 100\%     &86\%        &83\%      \\
        DQN on Player 2      &68\%& 33\%      &52\%        &31\%      \\
        \midrule
        PPO on Player 1      &55\%& 100\%     & 72\%       &76\%      \\
        PPO on Player 2      &54\%& 25\%      &47\%        &26\%      \\
        \midrule
        Random               &50\%& 20\%      &41\%        &21\%      \\
        \bottomrule
    \end{tabular}
    \label{tab:result3}
\end{table}

We further verify the RL approach on three asymmetric graphs, $G_2, G_3, G_4$. 
For $G_2$, to ensure a certain level of randomness, we initialize Player 1's resources across node 0 and node 1, and Player 2's across node 1 and node 2. To maintain a tie initial state, the number of resources allocated to node 1 by both players is kept identical. It is worth noting that since there is no path in $G_1$ leading to node 0, the setup inherently favors Player 1. This is because Player 1 can maintain permanent control over node 0 by allocating just 1 unit of resources there. 
We evaluate the performance of DQN and PPO against a random-policy agent. The results show that both DQN and PPO achieved a $100\%$ win rate against the random-policy agent. 
\begin{figure}
    \centering
    \subfloat[]{
    \label{g2}
    \includegraphics[width=0.8\linewidth]{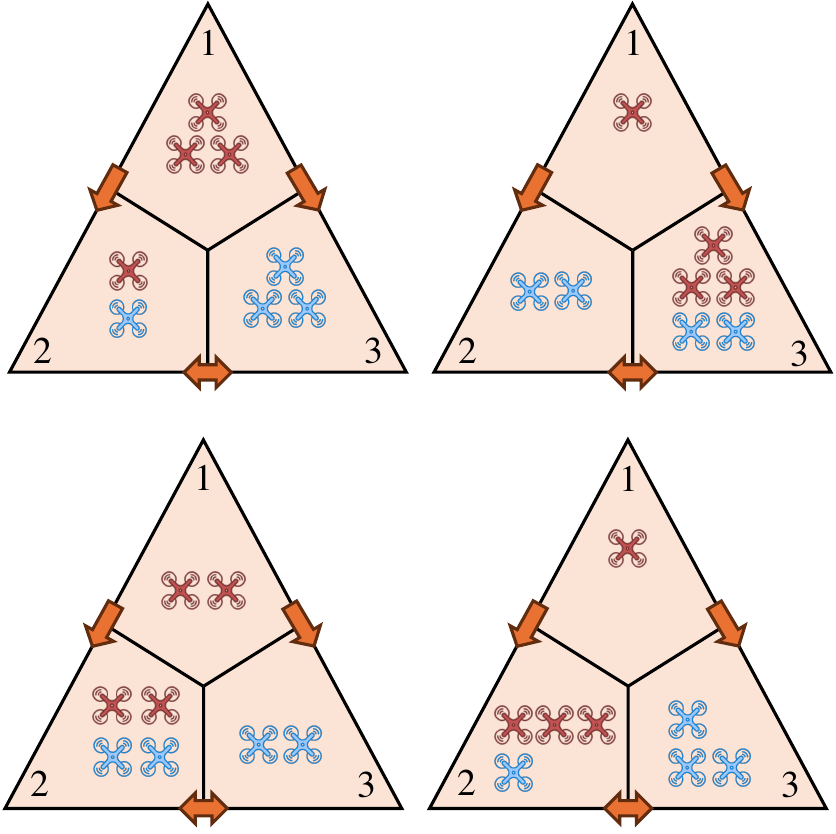}}
    
    \subfloat[]{
    \label{g4}
    \includegraphics[width=0.8\linewidth]{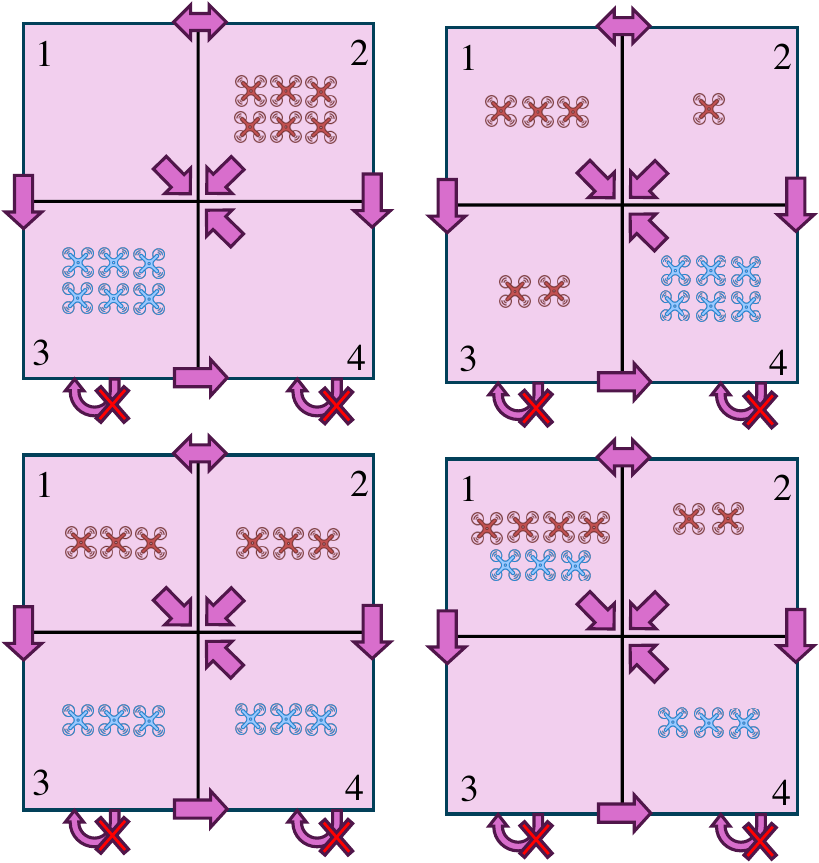}}
    \caption{Examples of gameplay on $G_1$ (\ref{g2}) and $G_3$ (\ref{g4}) where the two players have the same units of resources.}
    \label{m44}
\end{figure}
We extend a similar scenario to $G_3$, where Player 1's initial resources are randomly distributed between node 0 and node 1, and Player 2's initial resources are randomly distributed between node 2 and node 3. 
This initial resource distribution setup clearly continues to favor Player 1. 
Experimental results in Table~\ref{tab:result3} demonstrate that DQN and PPO perform significantly better in asymmetric graph $G_2$ compared to symmetric graph $G_0$. This indicates that DQN and PPO can take advantage of the graph structure. Specifically, DQN achieves an $86\%$ win rate, and PPO achieves a $71\%$ win rate, as presented in Table~\ref{tab:result3}.

On $G_4$, nodes 2 and 3 cannot retain any resources, specifically $\mathbf{H}^{G_3}_{22}=\mathbf{H}^{G_3}_{33}=0$. It means any resource on node 2 or node 3 at time step $t$ must be reallocated to other nodes at $t+1$. In addition, we randomly distribute Player 2's resources on nodes 2 and 3. Therefore, Player 2's action at the first step is predictable, so that $G_3$ is unfavorable to her.

As shown in Table~\ref{tab:result3}, DQN and PPO achieve a higher win rate compared with that in the symmetric case $G_0$. This again demonstrates DQN and PPO can take advantage of the asymmetric structure of the graph. 
These three cases demonstrate that both DQN and PPO can effectively leverage the structural advantages of the graph to improve their win rates.

Conversely, in all the scenarios above where the structure of the graph inherently puts Player 2 at a disadvantage, we also run simulations with Player 2 as the RL agent.
If both players follow random policies, their win rate is completely determined by the graph structure. Table~\ref{tab:result3} shows that if both players follow random policies, the win rate of Player 2 is only $20\%$ on $G_2$, $41\%$ on $G_3$, and $21\%$ on $G_4$. However, with DQN and PPO policies, the Player 2 win rate significantly improves to $33\%$ on $G_2$, $52\%$ on $G_3$, and $31\%$ on $G_4$ for PPO, and $25\%$ on $G_2$, $47\%$ on $G_3,$ and $26\%$ on $G_4$ for PPO. This verifies that DQN and PPO can still learn effective policies even on the graphs that are unfavorable.

Figure~\ref{m44} shows two example episodes (top and bottom rows) played on two asymmetric graphs: $G_2$ (Fig~\ref{g2}) and $G_4$ (Fig~\ref{g4}). In each episode, Player 1 (red) is a trained RL agent, while Player 2 (blue) follows a random policy. 
In both episodes on $G_1$, we observe that Player 1 consistently keeps a small portion of her resources on node 0, effectively securing permanent control of this isolated node. This strategic choice is reinforced by the structure of $G_2$, where node 0 is disconnected from the rest of the graph, making it unreachable for Player 2. The remaining resources are then used to contest Player 2 over the connected portion of the graph, resulting in a win for Player 1 in both cases.
On $G_4$, the structure imposes another type of asymmetry. Nodes 2 and 3 cannot retain resources across time steps, forcing any resources placed there to be reallocated immediately. Since Player 2’s resources are initially placed on these unstable nodes, its early actions become highly predictable. As shown in both episodes, Player 1 successfully exploits this predictability by strategically positioning its forces, eventually securing a majority and winning the game.
These visualizations demonstrate that the RL agent can exploit structural advantages in asymmetric graphs to outperform a random-policy opponent. This supports the claim that RL not only adapts to RL but also capitalizes on inherent graph biases to maximize winning potential.
\section{Conclusion}
In this paper, we leverage RL approaches, specifically DQN and PPO, to learn optimal policies for the GRAG. To ensure that actions taken during training adhere to graph constraints, we construct an action-displacement adjacency matrix and derive a valid action set from it. On the symmetric graph, the high win rates of DQN and PPO against both random-policy and greedy-policy agents demonstrate the effectiveness of RL in solving the GRAG. Additionally, a player trained using DQN against a deterministic-policy player achieves a high win rate against a random-policy player, highlighting a certain degree of generalization. On asymmetric graphs, DQN and PPO enable the advantaged player to maximize their strategic edge, often achieving near-perfect win rates. Conversely, while the disadvantaged player faces structural constraints or resource disadvantages, RL training still enhances their performance, resulting in a higher win rate compared to untrained strategies. This demonstrates that PPO and DQN can effectively exploit structural advantages while mitigating both structural and resource-related disadvantages.

Future research will explore several extensions of our RL-based framework for GRAG. First, we plan to evaluate alternative RL algorithms, including Double DQN \cite{van2016deep}, Dueling DQN \cite{wang2016dueling}, Prioritized Experience Replay (PER) \cite{schaul2015prioritized}, and Deep Deterministic Policy Gradient (DDPG) \cite{hou2017novel}, and conduct experiments across diverse multi-agent reinforcement learning (MARL) simulators. Second, we intend to introduce node-specific weights to reflect varying levels of strategic importance across the graph. This extension will require adapting our model to account for weighted nodes, and we aim to investigate whether RL methods can effectively optimize strategies under these conditions. Third, 
while resources are currently allocated in discrete units, we aim to extend our framework to handle scenarios where resources are treated as continuous variables. Finally, we aim to address the case of heterogeneous resources, where each player controls multiple resource types that exhibit mutual counteractions~\cite{an2023double}. This setting introduces additional strategic complexity, and we will explore the capability of RL techniques to manage resource interdependencies in this type of resource allocation game.
\bibliography{ref}

\begin{thebibliography}{10}
\providecommand{\url}[1]{#1}
\csname url@samestyle\endcsname
\providecommand{\newblock}{\relax}
\providecommand{\bibinfo}[2]{#2}
\providecommand{\BIBentrySTDinterwordspacing}{\spaceskip=0pt\relax}
\providecommand{\BIBentryALTinterwordstretchfactor}{4}
\providecommand{\BIBentryALTinterwordspacing}{\spaceskip=\fontdimen2\font plus
\BIBentryALTinterwordstretchfactor\fontdimen3\font minus \fontdimen4\font\relax}
\providecommand{\BIBforeignlanguage}[2]{{%
\expandafter\ifx\csname l@#1\endcsname\relax
\typeout{** WARNING: IEEEtran.bst: No hyphenation pattern has been}%
\typeout{** loaded for the language `#1'. Using the pattern for}%
\typeout{** the default language instead.}%
\else
\language=\csname l@#1\endcsname
\fi
#2}}
\providecommand{\BIBdecl}{\relax}
\BIBdecl

\bibitem{Ballotta2022ToCO}
\BIBentryALTinterwordspacing
L.~Ballotta, G.~Peserico, F.~Zanini, and P.~Dini, ``To compute or not to compute? adaptive smart sensing in resource-constrained edge computing,'' \emph{IEEE Transactions on Network Science and Engineering}, vol.~11, pp. 736--749, 2022. [Online]. Available: \url{https://api.semanticscholar.org/CorpusID:252090339}
\BIBentrySTDinterwordspacing

\bibitem{Fernando2024EnergyEfficientII}
\BIBentryALTinterwordspacing
X.~N. Fernando and G.~Lăzăroiu, ``Energy-efficient industrial internet of things in green 6g networks,'' \emph{Applied Sciences}, 2024. [Online]. Available: \url{https://api.semanticscholar.org/CorpusID:272909251}
\BIBentrySTDinterwordspacing

\bibitem{grocholsky2006cooperative}
B.~Grocholsky, J.~Keller, V.~Kumar, and G.~Pappas, ``Cooperative air and ground surveillance,'' \emph{IEEE Robotics \& Automation Magazine}, vol.~13, no.~3, pp. 16--25, 2006.

\bibitem{thakur2013planning}
D.~Thakur, M.~Likhachev, J.~Keller, V.~Kumar, V.~Dobrokhodov, K.~Jones, J.~Wurz, and I.~Kaminer, ``Planning for opportunistic surveillance with multiple robots,'' in \emph{2013 IEEE/RSJ International Conference on Intelligent Robots and Systems}.\hskip 1em plus 0.5em minus 0.4em\relax IEEE, 2013, pp. 5750--5757.

\bibitem{Rabban2020ImprovedRCA}
\BIBentryALTinterwordspacing
I.~E. Rabban and P.~Tokekar, ``Improved resilient coverage maximization with multiple robots,'' \emph{ArXiv}, vol. abs/2007.02204, 2020. [Online]. Available: \url{https://api.semanticscholar.org/CorpusId:220364130}
\BIBentrySTDinterwordspacing

\bibitem{Zhang2021GameTSA}
\BIBentryALTinterwordspacing
Z.~Zhang, J.~Smereka, J.~Lee, L.~Zhou, Y.~Sung, and P.~Tokekar, ``Game tree search for minimizing detectability and maximizing visibility,'' \emph{Autonomous Robots}, vol.~45, pp. 283 -- 297, 2021. [Online]. Available: \url{https://api.semanticscholar.org/CorpusId:231668346}
\BIBentrySTDinterwordspacing

\bibitem{Jia2018StackelbergGAA}
\BIBentryALTinterwordspacing
L.~Jia, Y.~Xu, Y.~Sun, S.~Feng, and A.~Anpalagan, ``Stackelberg game approaches for anti-jamming defence in wireless networks,'' \emph{IEEE Wireless Communications}, vol.~25, pp. 120--128, 2018. [Online]. Available: \url{https://api.semanticscholar.org/CorpusId:44149661}
\BIBentrySTDinterwordspacing

\bibitem{Alouini2024JammingIIA}
\BIBentryALTinterwordspacing
M.-S. Alouini, Y.~Ata, R.~Priyadarshani, and K.~Park, ``Jamming intrusions in extreme bandwidth communication: A comprehensive overview,'' \emph{ArXiv}, vol. abs/2403.19868, 2024. [Online]. Available: \url{https://api.semanticscholar.org/CorpusId:268793863}
\BIBentrySTDinterwordspacing

\bibitem{Zou2022EquilibriumAAA}
\BIBentryALTinterwordspacing
M.~Zou, J.~Chen, J.~Luo , Z.~Hu , and S.~Chen, ``Equilibrium approximating and online learning for anti-jamming game of satellite communication power allocation,'' \emph{Electronics}, 2022. [Online]. Available: \url{https://api.semanticscholar.org/CorpusId:253282601}
\BIBentrySTDinterwordspacing

\bibitem{Sen2010AnESA}
\BIBentryALTinterwordspacing
J.~Sen, ``An efficient security mechanism for high-integrity wireless sensor networks,'' \emph{ArXiv}, vol. abs/1111.0380, 2010. [Online]. Available: \url{https://api.semanticscholar.org/CorpusId:6682843}
\BIBentrySTDinterwordspacing

\bibitem{Wilden2022CyberNRA}
\BIBentryALTinterwordspacing
M.~K. Wilden, N.~Perra, T.~Eliassi-Rad, N.~Gozzi, S.~Boboila, A.~Oprea, P.~Angadi, J.~E. Loughner, and A.~Chernikova, ``Cyber network resilience against self-propagating malware attacks,'' \emph{ArXiv}, vol. abs/2206.13594, 2022. [Online]. Available: \url{https://api.semanticscholar.org/CorpusId:250089045}
\BIBentrySTDinterwordspacing

\bibitem{Wang2023HiBidACA}
\BIBentryALTinterwordspacing
H.~Wang, B.~Tang, C.~H. Liu, S.~Mao, J.~Zhou, Z.~Dai, Y.~Sun, Q.~Xie, X.~Wang, and D.~Wang, ``Hibid: A cross-channel constrained bidding system with budget allocation by hierarchical offline deep reinforcement learning,'' \emph{IEEE Transactions on Computers}, vol.~73, pp. 815--828, 2023. [Online]. Available: \url{https://api.semanticscholar.org/CorpusId:266292412}
\BIBentrySTDinterwordspacing

\bibitem{Zhang2023FairnessAwareCBA}
\BIBentryALTinterwordspacing
C.~Zhang, J.~Zhou, J.~Wang, J.~Fan, and Y.~Shi, ``Fairness-aware competitive bidding influence maximization in social networks,'' \emph{IEEE Transactions on Computational Social Systems}, vol.~11, pp. 2147--2159, 2023. [Online]. Available: \url{https://api.semanticscholar.org/CorpusId:259539732}
\BIBentrySTDinterwordspacing

\bibitem{borel1921theorie}
E.~Borel, ``La th{\'e}orie du jeu et les {\'e}quations int{\'e}gralesa noyau sym{\'e}trique,'' \emph{Comptes rendus de l’Acad{\'e}mie des Sciences}, vol. 173, no. 1304-1308, p.~58, 1921.

\bibitem{borel1953theory}
------, ``The theory of play and integral equations with skew symmetric kernels,'' \emph{Econometrica: journal of the Econometric Society}, pp. 97--100, 1953.

\bibitem{shishika2022dynamic}
D.~Shishika, Y.~Guan, M.~Dorothy, and V.~Kumar, ``Dynamic defender-attacker blotto game,'' in \emph{2022 American Control Conference (ACC)}.\hskip 1em plus 0.5em minus 0.4em\relax IEEE, 2022, pp. 4422--4428.

\bibitem{an2023double}
Z.~An and L.~Zhou, ``Double oracle algorithm for game-theoretic robot allocation on graphs,'' \emph{IEEE Transactions on Robotics, Accepted}, 2025.

\bibitem{schrittwieser2020mastering}
J.~Schrittwieser, I.~Antonoglou, T.~Hubert, K.~Simonyan, L.~Sifre, S.~Schmitt, A.~Guez, E.~Lockhart, D.~Hassabis, T.~Graepel \emph{et~al.}, ``Mastering atari, go, chess and shogi by planning with a learned model,'' \emph{Nature}, vol. 588, no. 7839, pp. 604--609, 2020.

\bibitem{silver2017mastering}
D.~Silver, T.~Hubert, J.~Schrittwieser, I.~Antonoglou, M.~Lai, A.~Guez, M.~Lanctot, L.~Sifre, D.~Kumaran, T.~Graepel \emph{et~al.}, ``Mastering chess and shogi by self-play with a general reinforcement learning algorithm,'' \emph{arXiv preprint arXiv:1712.01815}, 2017.

\bibitem{silver2018general}
------, ``A general reinforcement learning algorithm that masters chess, shogi, and go through self-play,'' \emph{Science}, vol. 362, no. 6419, pp. 1140--1144, 2018.

\bibitem{lai2015giraffe}
M.~Lai, ``Giraffe: Using deep reinforcement learning to play chess,'' \emph{arXiv preprint arXiv:1509.01549}, 2015.

\bibitem{madani2014game}
K.~Madani and M.~Hooshyar, ``A game theory--reinforcement learning (gt--rl) method to develop optimal operation policies for multi-operator reservoir systems,'' \emph{Journal of Hydrology}, vol. 519, pp. 732--742, 2014.

\bibitem{wender2012applying}
S.~Wender and I.~Watson, ``Applying reinforcement learning to small scale combat in the real-time strategy game starcraft: Broodwar,'' in \emph{2012 IEEE conference on computational intelligence and games (cig)}.\hskip 1em plus 0.5em minus 0.4em\relax IEEE, 2012, pp. 402--408.

\bibitem{ye2020towards}
D.~Ye, G.~Chen, W.~Zhang, S.~Chen, B.~Yuan, B.~Liu, J.~Chen, Z.~Liu, F.~Qiu, H.~Yu \emph{et~al.}, ``Towards playing full moba games with deep reinforcement learning,'' \emph{Advances in Neural Information Processing Systems}, vol.~33, pp. 621--632, 2020.

\bibitem{ye2019deep}
H.~Ye, G.~Y. Li, and B.-H.~F. Juang, ``Deep reinforcement learning based resource allocation for v2v communications,'' \emph{IEEE Transactions on Vehicular Technology}, vol.~68, no.~4, pp. 3163--3173, 2019.

\bibitem{cui2019multi}
J.~Cui, Y.~Liu, and A.~Nallanathan, ``Multi-agent reinforcement learning-based resource allocation for uav networks,'' \emph{IEEE Transactions on Wireless Communications}, vol.~19, no.~2, pp. 729--743, 2019.

\bibitem{deng2020dynamical}
S.~Deng, Z.~Xiang, P.~Zhao, J.~Taheri, H.~Gao, J.~Yin, and A.~Y. Zomaya, ``Dynamical resource allocation in edge for trustable internet-of-things systems: A reinforcement learning method,'' \emph{IEEE Transactions on Industrial Informatics}, vol.~16, no.~9, pp. 6103--6113, 2020.

\bibitem{noel2022reinforcement}
J.~C.~G. Noel, ``Reinforcement learning agents in colonel blotto,'' \emph{arXiv preprint arXiv:2204.02785}, 2022.

\bibitem{mnih2015human}
V.~Mnih, K.~Kavukcuoglu, D.~Silver, A.~A. Rusu, J.~Veness, M.~G. Bellemare, A.~Graves, M.~Riedmiller, A.~K. Fidjeland, G.~Ostrovski \emph{et~al.}, ``Human-level control through deep reinforcement learning,'' \emph{nature}, vol. 518, no. 7540, pp. 529--533, 2015.

\bibitem{schulman2017proximal}
J.~Schulman, F.~Wolski, P.~Dhariwal, A.~Radford, and O.~Klimov, ``Proximal policy optimization algorithms,'' \emph{arXiv preprint arXiv:1707.06347}, 2017.

\bibitem{tokic2010adaptive}
M.~Tokic, ``Adaptive $\varepsilon$-greedy exploration in reinforcement learning based on value differences,'' in \emph{Annual conference on artificial intelligence}.\hskip 1em plus 0.5em minus 0.4em\relax Springer, 2010, pp. 203--210.

\bibitem{terry2021pettingzoo}
J.~Terry, B.~Black, N.~Grammel, M.~Jayakumar, A.~Hari, R.~Sullivan, L.~S. Santos, C.~Dieffendahl, C.~Horsch, R.~Perez-Vicente \emph{et~al.}, ``Pettingzoo: Gym for multi-agent reinforcement learning,'' \emph{Advances in Neural Information Processing Systems}, vol.~34, pp. 15\,032--15\,043, 2021.

\bibitem{tianshou}
\BIBentryALTinterwordspacing
J.~Weng, H.~Chen, D.~Yan, K.~You, A.~Duburcq, M.~Zhang, Y.~Su, H.~Su, and J.~Zhu, ``Tianshou: A highly modularized deep reinforcement learning library,'' \emph{Journal of Machine Learning Research}, vol.~23, no. 267, pp. 1--6, 2022. [Online]. Available: \url{http://jmlr.org/papers/v23/21-1127.html}
\BIBentrySTDinterwordspacing

\bibitem{van2016deep}
H.~Van~Hasselt, A.~Guez, and D.~Silver, ``Deep reinforcement learning with double q-learning,'' in \emph{Proceedings of the AAAI conference on artificial intelligence}, vol.~30, no.~1, 2016.

\bibitem{wang2016dueling}
Z.~Wang, T.~Schaul, M.~Hessel, H.~Hasselt, M.~Lanctot, and N.~Freitas, ``Dueling network architectures for deep reinforcement learning,'' in \emph{International conference on machine learning}.\hskip 1em plus 0.5em minus 0.4em\relax PMLR, 2016, pp. 1995--2003.

\bibitem{schaul2015prioritized}
T.~Schaul, J.~Quan, I.~Antonoglou, and D.~Silver, ``Prioritized experience replay,'' \emph{arXiv preprint arXiv:1511.05952}, 2015.

\bibitem{hou2017novel}
Y.~Hou, L.~Liu, Q.~Wei, X.~Xu, and C.~Chen, ``A novel ddpg method with prioritized experience replay,'' in \emph{2017 IEEE international conference on systems, man, and cybernetics (SMC)}.\hskip 1em plus 0.5em minus 0.4em\relax IEEE, 2017, pp. 316--321.

\end{thebibliography}
\bibliographystyle{IEEEtran}


\vfill

\end{document}